%% file: main.tex
\documentclass{article}

\usepackage{microtype}
\usepackage{graphicx}
\usepackage{subfigure}
\usepackage{booktabs} %

\usepackage{hyperref}

\usepackage[accepted]{arxiv}

\input{commands}

\icmltitlerunning{Which transformer architecture fits my data? A vocabulary bottleneck in self-attention}
\newcommand*{\ARXIV}{}

\begin{document}
	
	\twocolumn[
	\icmltitle{Which Transformer architecture fits my data? \\A vocabulary bottleneck in self-attention
	}

	\icmlsetsymbol{equal}{*}
	
	\begin{icmlauthorlist}
		\icmlauthor{Noam Wies}{huji}
		\icmlauthor{Yoav Levine}{huji}
		\icmlauthor{Daniel Jannai}{huji}
		\icmlauthor{Amnon Shashua}{huji}
	\end{icmlauthorlist}

	\icmlaffiliation{huji}{The Hebrew University of Jerusalem}
	
	\icmlcorrespondingauthor{Noam Wies}{noam.wies@cs.huji.ac.il}
	
	\icmlkeywords{Transformer, Self-Attention, Vocabulary, VisionTransformer, Expressivity}
	
	\vskip 0.3in
	]

	\printAffiliationsAndNotice{} %
	
	\input{body}
	\bibliography{main}
	\bibliographystyle{icml2021}
	
	\onecolumn
	\appendix
	\tableofcontents
	\input{appendix_body}

\end{document}

%% file: commands.tex
\usepackage{amsfonts}
\usepackage{nicefrac}
\usepackage{amsmath}
\usepackage{mathtools}
\usepackage{amsthm}
\usepackage{stmaryrd}
\newcommand{\aaa}{{\mathbf a}}
\newcommand{\bb}{{\mathbf b}}
\newcommand{\A}{{\mathcal A}}
\newcommand{\B}{{\mathcal B}}
\newcommand{\x}{{\mathbf x}}
\newcommand{\y}{{\mathbf y}}
\newcommand{\uu}{{\mathbf u}}
\newcommand{\p}{{\mathbf p}}

\newcommand{\eg}{\emph{e.g.}}
\newcommand{\ie}{\emph{i.e.}}
\newcommand{\wrt}{w.r.t.}
\newcommand{\R}{{\mathbb R}}

\newcommand{\V}{{\mathbf V}}
\newcommand{\N}{{\mathbb N}}
\newcommand{\nocontentsline}[3]{}
\newcommand{\tocless}[2]{\bgroup\let\addcontentsline=\nocontentsline#1{#2}\egroup}
\newcommand{\sep}[2]{\mathrm{sep}_{(#1)}\left( #2 \right)}
\newcommand{\rank}[1]{\mathrm{rank}\left( #1 \right)}
\newcommand{\abs}[1]{\left\lvert#1 \right\rvert}
\newcommand{\mat}[1]{\llbracket#1\rrbracket}
\def\bea{\begin{eqnarray}}
\def\eea{\end{eqnarray}}

\newtheorem{lemma}{Lemma}
\newtheorem{corollary}{Corollary}
\newtheorem{theorem}{Theorem}

\newtheorem{claim}{Claim}

\def\multiset#1#2{\ensuremath{\left(\kern-.3em\left(\genfrac{}{}{0pt}{}{#1}{#2}\right)\kern-.3em\right)}}
\newcommand{\cupdot}{\mathbin{\mathaccent\cdot\cup}}

\newcommand{\wk}{{W^{\text{conv}}}}

\usepackage{url}
\def\simleq{\; \raise0.3ex\hbox{$<$\kern-0.75em \raise-1.1ex\hbox{$\sim$}}\; }
\def\simgeq{\; \raise0.3ex\hbox{$>$\kern-0.75em \raise-1.1ex\hbox{$\sim$}}\; }

%% file: body.tex
	\begin{abstract}
	After their successful debut in natural language processing, Transformer architectures are now becoming the de-facto standard in many domains. An obstacle for their deployment over new modalities is the architectural configuration: the optimal depth-to-width ratio has been shown to dramatically vary across data types (\eg, $10$x larger over images than over language). We theoretically predict the existence of an embedding rank bottleneck that limits the contribution of self-attention width to the Transformer expressivity. We thus directly tie the input vocabulary size and rank to the optimal depth-to-width ratio, since a small vocabulary size or rank dictates an added advantage of depth over width. We empirically demonstrate the existence of this bottleneck and its implications on the depth-to-width interplay of Transformer architectures, linking the architecture variability across domains to the often glossed-over usage of different vocabulary sizes or embedding ranks in different domains. 
	As an additional benefit, our rank bottlenecking framework allows us to identify size redundancies of $25\%-50\%$ in leading NLP models such as ALBERT and T5.    
\end{abstract}

\ifdefined\SQUEEZE \vspace{-3mm} \fi
\tocless\section{Introduction \label{sec:Introduction}}

Since the introduction of the Transformer as a sequence-to-sequence model for machine translation, its variants have achieved state-of-the-art results in various domains, such as text \cite{devlin2018bert}, images \cite{chen2020generative,dosovitskiy2021an,child2020distributionaugmentation}, audio \cite{dhariwal2020jukebox,baevski2020wav2vec}, video \cite{Weissenborn2020Scaling}, mathematical problem solving \cite{saxton2018analysing}, reinforcement learning \cite{vinyals2019grandmaster,chen2021decisiontransformer} and bioinformatics \cite{rives2019biological,rao2021msa}. While the architecture's operation is mostly unchanged, the chosen ratio between the number of self-attention layers (depth) and the dimension of the internal representation (width) varies greatly across different applications.
For example, for a fixed BERT-Base size of $110$M parameters, popular architectures range from $12$-layered networks %
to much narrower $128$-layered networks. %

In language applications, the depth-to-width ratio of Transformer models is relatively consistent: increase in model size is mainly done via widening~\cite{levine2020limits}, so that the largest self-attention architectures are very wide relative to their depths~\citep{raffel2019exploring,GPT3}.  In other domains this aspect seems unresolved, even when comparing leading models within the same field.
In computer vision, for example, the Vision Transformer (ViT) \cite{dosovitskiy2021an}  sets the state-of-the-art on ImageNet in a transfer learning setup with depth-to-width ratios corresponding to common language models. Conversely, Image GPT \cite{chen2020generative} and Sparse Transformer \cite{child2020distributionaugmentation} achieve state-of-the-art results in unsupervised learning and density estimation, respectively, by using significantly deeper and narrower models. 

Recently, \citet{henighan2020scaling} perform an ablation study which includes data from different domains, and report empirical evidence for the existence of different ``ideal" depth-to-width ratios per data modality. 
Figure~4 in that study leads the authors to conclude that the depth-to-width ratio of image and math models should be $10$x larger than that of language models. 
A possible take-away can be that the representations of the different data types require different depth-to-width specifications from the architecture.	
In a contemporary study, \citet{levine2020limits}
quantify the optimal depth-to-width ratio per self-attention network size. 
Their results justify the relative shallowness of current attention-based language models, and moreover suggest that further deepening should be logarithmic in widening.   
Importantly, their theoretical framework pertains to the self-attention architecture expressivity and is agnostic to the input modality.

\begin{figure*}[t]
	\vskip 0.2in
	\begin{center}
		\ifdefined\ARXIV
		\centerline{\includegraphics[width=\linewidth]{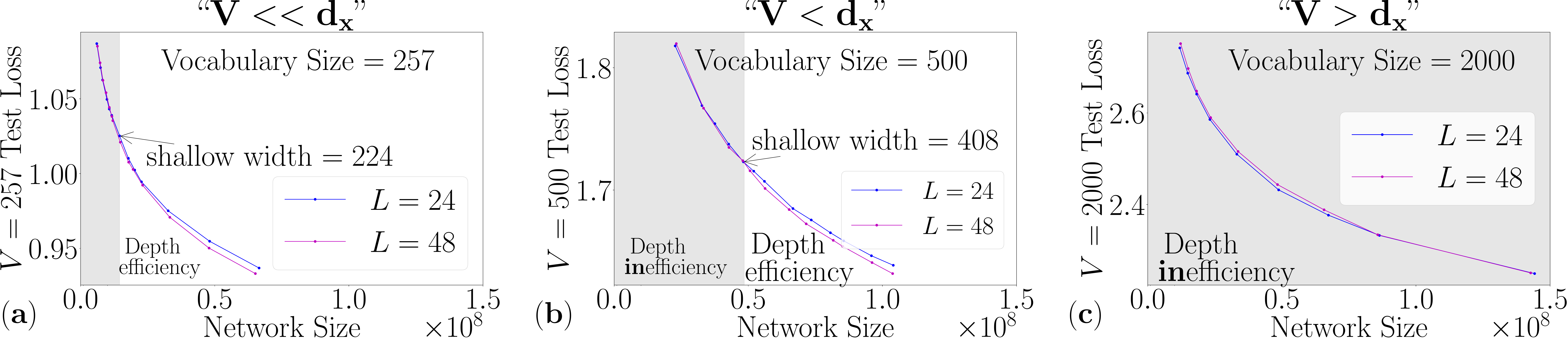}}
		\else
		\centerline{\includegraphics[width=\linewidth]{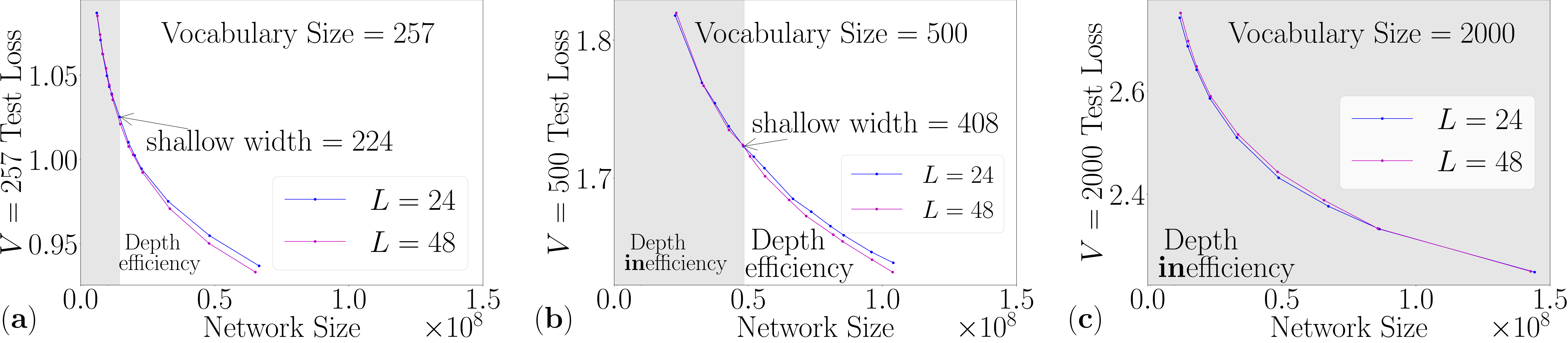}}
		\fi		
		\caption{
			An experimental validation of the effect of vocabulary size on the depth-to-width tradeoff in transformers (full details in section~\ref{sec:experiments:implications_to_depth_width_interplay}).  
			The experiments were performed over text, but demonstrate that similarly to images, when using a byte level vocabulary to represent text, deeper networks outperform shallower ones in practical network sizes (BERT-Base is of size $\sim100$M). 
			Figures~(a) and~(b) use similar vocabulary sizes to those used by \citet{henighan2020scaling} and~\citet{child2020distributionaugmentation} over images, respectively. Both studies operated with significantly deeper transformers than commonly used over text.
			Figure (c) uses a vocabulary size that is larger than the network width, as commonly done for language models, and the advantage of the deeper networks disappears. 
			Note that the loss value depends on the vocabulary size, and thus loss values between figures (a)-(c) are not directly comparable.}
		\label{fig:text_vocab_257_vs_2000}
	\end{center}
	\vskip -0.2in
\end{figure*}

The two different views above give rise to the following  question: \textbf{does the optimal depth-to-width ratio depend on the data modality, or is it a consequence of architecture expressivity considerations?}  
In this paper, we establish architecture expressivity results that explain the observed variability in the architecture configurations across domains. 
By identifying expressivity bottlenecks that have manifested themselves mainly in non-linguistic data modalities, we provide simple \textit{domain independent} practical guidelines for Transformer architecture selection. 

Specifically, we identify a \textit{vocabulary bottleneck} in self-attention, proving that the rank of the input embedding matrix caps the network width's  contribution to expressivity. 
Furthermore, we show that when the width exceeds the input embedding rank, deepening is exponentially favorable over widening.
We empirically demonstrate the predicted effects of the vocabulary bottleneck.
Figure~\ref{fig:text_vocab_257_vs_2000} shows experiments in the language domain, demonstrating that when decreasing the vocabulary size below the value of the network width, depth becomes more important than width also in this data modality. This provides empirical support for our expressivity-based interpretation of the variation in the depth-to-width ratio, and counters the modality based interpretation.

As the use of Transformer architectures was extended to different modalities, the effect of this architectural element was overlooked.
For example, in the bioinformatics domain, \citet{rives2019biological} lead the RaptorX benchmark for protein long-range contact precision~\cite{wang2017RaptorX} with a Transformer model that has width $1280$, but a vocabulary size $33$, equivalent to a character-level model. 
Our theoretical results, backed by targeted experiments, indicate that this is very sub-optimal, as the width is severely capped by the low rank of the embedding matrix. In vision, the above-mentioned depth-to-width ratio inconsistency across models is commensurate with our findings --  the depth-$128$ Sparse Transformer model has a pixel-intensity vocabulary size of $256$, while the depth-$12$  ViT-Base model enjoys full-rank embedding and can utilize its width of $768$. 

The above vocabulary bottleneck rarely comes into play when considering language-related applications -- the input embedding matrix is commonly a fully-ranked matrix with dimensions of the words vocabulary size ($\sim30$K) times the network width ($\sim1$K). 
However, the popular ALBERT model of~\citet{Lan2020ALBERT} has deliberately reduced the rank of the embedding matrix for space efficiency reasons. 
Our results formalize the negative impact of this approach on model expressivity, and our experiments quantify the performance degradation --   the ALBERT approach leads to $25\%$ redundancy in network size, \ie, a low rank network is surpassed by a full rank network with $75\%$ of the parameters. 

Another consequence of our framework that applies to leading language models, has to do with the method used in \citet{raffel2019exploring} for scaling up to their $11$B parameter language model, referred to as T5-11B (which holds the state-of-the-art in many NLP benchmarks).  
Due to hardware related considerations, they elected to keep the representation dimension between layers relatively small, and to invest most of the parameters in the self-attention operation itself. 
Beyond the vocabulary bottleneck, our rank bottlenecking proof technique applies to bottlenecks created mid-architecture, and specifically we show that a low representation dimension caps the ability to enjoy an excessive parameter increase in the self-attention operation. We validate this prediction empirically, and project T5-11B to be $\sim45\%$ redundant, \ie, it could achieve its performance with roughly half its size if trained with a regular architecture. Notably, a modified version of T5-11B, dubbed T51.1 XXL~\cite{craffel2021t511}, fixed the above bottleneck in a manner which completely accords with our recommendation. We expand on this modification in section~\ref{sec:experiments:t5_bottleneck}.

The remainder of this paper is organized as follows. In section~\ref{sec:TransformerArchitecture} we present the analyzed input embedding methods and Transformer architecture. In section~\ref{sec:separation_rank} we present a measure referred to as a function's separation rank, with which, in section~\ref{sec:The_Low_Rank_Bottleneck}, we establish our main results regarding functions realized by Transformer architectures. 
In section~\ref{sec:experiments}, we empirically demonstrate the predicted bottleneck phenomena, summarized by:
\begin{enumerate}
	\vspace{-2mm}
	\item 		A  degradation when the input embedding rank is smaller than the network width (section~\ref{sec:experiments:absoulte_degradation_exp});
	\vspace{-2mm}
	\item 	An advantage of depth over width when the input embedding rank is smaller than the network width	(section~\ref{sec:experiments:implications_to_depth_width_interplay});	
	\vspace{-2mm}
	\item A degradation when the network width is smaller than the internal attention representation	(section~\ref{sec:experiments:t5_bottleneck}).
\end{enumerate} 

\ifdefined\SQUEEZE \vspace{-5mm} \fi
\tocless\section{The analyzed Transformer architecture \label{sec:TransformerArchitecture}}

While the original encoder-decoder based Transformer architecture of \citet{vaswani2017transformer} is still widely used, variants based solely on its encoder (\eg, BERT~\cite{devlin2018bert}) or decoder (\eg, GPT~\cite{radford2018improving}) have gained popularity in various domains. For the sake of simplicity, we will analyze such variants. 
The analyzed Transformer architecture is comprised of an input embedding layer, which we present in section~\ref{sec:TransformerArchitecture:Embedding}, followed by $L$ Transformer layers, which we present subsequently in section~\ref{sec:TransformerArchitecture:Arch}. 
\tocless\subsection{The input embedding layer \label{sec:TransformerArchitecture:Embedding}}
We will analyze two common methods for translating raw data into an embedding. 
The first method, \textit{vocabulary} based, is employed when the input is a sequence of $N$ variables that can have one of $V$ discrete values, referred to as vocabulary tokens, \ie, $\left\{w^i\right\}_{i=1}^{N}$ where $\forall i:~w^i\in[V]$. 
Naturally, this method is prevalent in language applications of Transformers (hence the name), 
but it appears also in other domains such as bioinformatics~\cite{rives2019biological} or computer vision~\cite{chen2020generative}. 
The second embedding method that we analyze is \textit{convolution} based, used for example over images (Vision Transformer \cite{dosovitskiy2021an}). %
This method is applied when the raw input sequence is of real valued vectors $\left\{\x^i\right\}_{i=1}^{M}$, where $\forall i: \x^i\in\R^{d_\text{input}}$, and down-sampling is required in order to reduce the sequence length related computation costs.

When using the vocabulary based embedding, since the size of the vocabulary affects storage and runtime, it is common for $V$ to reflect a precomputed compression of the  ``raw vocabulary" of all possible symbols in the raw data. For example, in language, while character level vocabularies could be used with manageable costs (though as we prove below, these incur severe underutilization of the Transformer's expressive power), when aiming for word-level vocabularies, the number of unique words in web based sources can surpass $1$M, and therefore sub-word based encodings such as  WordPiece~\cite{wordpiece}, SentencePiece~\cite{kudo-richardson-2018-sentencepiece} or BPE~\cite{sennrich-etal-2016-neural} are used to reduce the number of vocabulary tokens to~$V\sim30$K.
In computer vision, for ImageGPT,~\citet{chen2020generative} pre-compute $512$ clusters over $3$-dimentional RGB pixel values, along with a mapping of each pixel to one of corresponding ${V}=512$ vocabulary tokens.%

Translating the input sequence $\left\{w^i\in[V]\right\}_{i=1}^{N}$ into indicators: $\left\{\hat{\mathbf{w}}^i=\hat{e}_{w^i}\right\}_{i=1}^{N}$, where $\forall i:\hat{\mathbf{w}}^i\in\V:=\R^{V}$, the output of the embedding layer at position $i\in[N]$ is:
\begin{equation}\label{eq:vocab}
	\y^{0,i}=M_{\textrm{V}} ~\hat{\mathbf{w}}^i+\p^i,
\end{equation}
where $M_{\textrm{V}}\in\R^{d_x\times V}$ is a learned matrix referred to as the vocabulary matrix (or the embedding matrix). Accordingly, $\y^{0,i}$ is a vector in $\R^{d_x}$, \ie, per location, the input to the first Transformer layer is of dimension $d_x$, the network width. The added learned position dependent term $\p^i\in\R^{d_x}$ is referred to as the positional embedding.

In the second, convolution based, input embedding method, in order to fit $M$ real valued input vectors into a Transformer layer with an input sequence of size $N$ (such that $M$ is a multiple integer of $N$), a
convolutional kernel  $\wk\in\R^{\frac{M}{N}\times d_x\times d_{\text{input}}}$ is used for computing the $i$th output of the embedding layer:

\ifdefined\SQUEEZE \vspace{-9mm} \fi
\begin{align}\label{eq:conv}
	\y^{0,i}
	=\sum_{j=1}^{\frac{M}{N}}W^{\text{conv}}_{j}\x^{\frac{M}{N}\cdot (i-1)+j}+\p^i,
\end{align}
\ifdefined\SQUEEZE \vspace{-2mm} \fi
where $\p^i$ is the added positional embedding.

In section~\ref{sec:The_Low_Rank_Bottleneck},  we will show an expressivity bottlenecking result that depends on $r$, a measure of rank that corresponds to the employed embedding method. For the vocabulary embedding method,
$r\leq \min\{d_x,V\}$ is the rank of the vocabulary matrix: 
\ifdefined\SQUEEZE \vspace{-2mm} \fi
\begin{equation}\label{eq:vocab_r}
	r=\rank{M_{\textrm{V}}}.
\end{equation}
For the convolution method, by defining the effective vocabulary dimension to be $V:=\nicefrac{M}{N}\cdot d_{\textrm{input}}$, we reshape the convolutional kernel $W^{\textrm{conv}}$ into a matrix $\tilde{W}^{\textrm{conv}}\in\R^{d_x\times V}$ and define $r\leq \min\{d_x,V\}$ as:
\ifdefined\SQUEEZE \vspace{-1mm} \fi
\begin{equation}\label{eq:conv_r}
	r=\rank{\tilde{W}^{\textrm{conv}}}.
\end{equation} 
\ifdefined\SQUEEZE \vspace{-5mm} \fi

Importantly, though the data modality influences embedding considerations, there is relative freedom in choosing the embedding method (\eg, both the above methods were employed over images) and controlling the corresponding embedding rank $r$ (via choosing the vocabulary size/rank or the convolution kernel size). Therefore, we will argue that by noticing the rank related expressivity bottleneck, suited input embeddings that allow for full utilization of the Transformer expressivity can be used in different domains.

\ifdefined\SQUEEZE \vspace{-2mm} \fi
\tocless\subsection{The self-attention architecture \label{sec:TransformerArchitecture:Arch}}
\ifdefined\SQUEEZE \vspace{-1mm} \fi
Following \cite{levine2020limits}, our theoretical analysis will focus on a variant of self-attention in which the attention scores are unnormalized, and which compounds self-attention layers without mixing in element-wise feed-forward layers in between. Accordingly, given an embedding output sequence $\{\y^{0,i}\}_{i=1}^N$ (see previous subsection), the function realized by the analyzed $H$-headed depth-$L$ width-$d_x$ Transformer architecture is recursively written: 

\ifdefined\SQUEEZE \vspace{-7mm} \fi
\begin{align}\label{eq:our_layer}\y^{l+1,i}\left(\y^{l,1},...,\y^{l,N}\right)\coloneqq\sum_{h=1}^{H}W^{\textrm{O},l,h}\sum_{j=1}^{N}a_{hj}^{i}W^{\textrm{V},l,h}\y^{l,j}\\
	a_{hj}^{i}\coloneqq\left\langle W^{\textrm{Q},l,h}\y^{l,i},W^{\textrm{K},l,h}\y^{l,j}\right\rangle~~~~~~&\nonumber
\end{align}
\ifdefined\SQUEEZE \vspace{-3mm} \fi

where $\forall l\in[L] ,h\in[H]$, $W^{\textrm{K},l,h},W^{\textrm{Q},l,h},W^{\textrm{V},l,h}$, $\left(W^{\textrm{O},l,h}\right)^{\top}\in\R^{d_a\times d_x}$ are the Key, Query, Value and Output self-attention learned weights matrices introduced in~\citet{vaswani2017transformer}. The attention dimension is usually chosen as $d_a=\nicefrac{d_x}{H}$, but in section~\ref{sec:The_Low_Rank_Bottleneck:implications:t5} we analyze a bottleneck related to relaxing this constraint, and in section~\ref{sec:experiments:t5_bottleneck} we project that due to this bottleneck the leading T5 architecture of~\citet{raffel2019exploring} could be almost half of its current size and reach the same performance.

The above relaxations are justified by noting that:
\begin{enumerate}
	\ifdefined\SQUEEZE \vspace{-2mm} \fi
	\item\citet{press2019improving} train a ``self-attention first"  network that first performs all of the  self-attention operations consecutively, and only then performs all of the position-wise feed-forward operations. This network achieves comparable language modeling performance relatively to the regular approach of interleaving these functionalities. 
	Since the feed-forward operation does not mix different locations, this outcome directly implies that the self-attention mechanism itself provides all of the elaborate input integration, and that the interleaved feed-forward layer is just a performance booster.
	\ifdefined\SQUEEZE \vspace{-2mm} \fi
	\item While removing the attention score normalization via softmax is not in line with an intuitive interpretation of attention as distributing ``fractions" of an overall attention budget among inputs, a growing body of work shows that the attention weights distribution does not directly correlate with predictions~\citep{AttentionIsNotExplanation,pruthi2019learning,brunner2020identifiability}.
	Moreover, \cite{richter2020normalized} recently point out undesirable traits of the softmax operation, demonstrating that its property of confining the outcome to the convex hull of its inputs unnecessarily limits the expressibility of the self-attention mechanism. 
	The analyzed unnormalized variant retains the actual operation of dynamically linking input and output locations via the Key/Query/Value connectivity of self-attention. 
\end{enumerate}
\ifdefined\SQUEEZE \vspace{-2mm} \fi

The goal of the above points is not to advocate modifications in the Transformer's non-linearity or normalization operations, but to note that while these are under examination and are susceptible to alteration, the connectivity of self-attention, manifested by eq.~\eqref{eq:our_layer}, is the core mechanism driving its functionality. 
A reinforcing signal to the above argument is the relevance of conclusions drawn by directly analyzing the self-attention mechanism to experiments in commonly employed self-attention networks, as presented in~\citet{levine2020limits} regarding depth efficiency regimes as well by us later in section~\ref{sec:experiments} regarding the effects of the embedding rank.
These experiments are consistently compatible with theoretical predictions that arise from our framework.

\ifdefined\SQUEEZE \vspace{-1mm} \fi
\tocless\section{A capacity for modeling input dependencies \label{sec:separation_rank}}
\ifdefined\SQUEEZE \vspace{-0mm} \fi
In this section, we introduce the separation rank of the function realized by a Transformer as a measure that quantifies its ability to model dependencies between subsets of its inputs.
We will use this measure in section~\ref{sec:The_Low_Rank_Bottleneck} in order to establish the embedding bottleneck in Transformer architectures.
The separation rank, introduced in \citet{beylkin2002numerical} for high-dimensional numerical analysis, 
was employed for various applications, \eg,~chemistry~\cite{harrison2003multiresolution}, particle engineering~\cite{hackbusch2006efficient}, and machine learning~\cite{beylkin2009multivariate}. 
More recently,
the separation rank has been established as a measure of dependencies modeled by deep convolutional and recurrent networks \wrt~their inputs~\citep{cohen2017inductive,cohen2017analysis,levine2018benefits}, and tied to quantum entanglement measures for proving that these deep learning architectures can model elaborate many-body quantum particle correlations~\cite{levine2018deep,levine2019quantum,sharir2020deep}. 
Our usage of the separation rank directly follows that in \citet{levine2020limits}, who employed this measure for studying the depth-to-width interplay in self-attention networks.

Let $(A,B)$ be a balanced partition of the input locations, \ie,~$A$ and~$B$ are equal sized disjoint subsets of~$[N]$ whose union gives~$[N]$.
The separation rank of a function $y(\x^1,\ldots,\x^N)$~\wrt~a partition $(A,B)$, is the minimal number of summands that together sum up to equal $y$, where each summand is \emph{multiplicatively separable \wrt~$(A,B)$}, \ie,~is equal to a product of two functions~--~one that intakes only inputs from one subset $\{\x^{i}:i\in A\}$, and another that intakes only inputs from the other subset $\{\x^{i}:i\in B\}$. 
Formally, the \emph{separation rank} of $y:\V^N\to\R$ \wrt~the partition $(A,B)$ is defined as follows:
\ifdefined\SQUEEZE \vspace{-1mm} \fi
\begin{align*}
	sep&\left(y,A,B\right)  \coloneqq\min\{R\in\N\cup\left\{ 0\right\} :\\
	& \exists g_{1},\dots,g_{R}:\V^{N/2}\to\R,g'_{1},\dots,g'_{R}:\V^{N/2}\to\R\\
	& y\left(\x^1,\ldots,\x^N\right)=\\
	& \sum\nolimits_{r=1}^{R}g_{r}\left(\{\x^{i}:i\in A\}\right)g'_r\left(\{\x^{i}:i\in B\}\right)\}
\end{align*}
\ifdefined\SQUEEZE \vspace{-4mm} \fi

We will use the separation rank as a quantifier of correlations that can be expressed by the model.
If the separation rank of a function \wrt~an input partition is~$1$, the function is separable, meaning it cannot take into account consistency between $\{\x^{i}\}_{i\in A}$ and $\{\x^{i}\}_{i\in B}$.
In a statistical setting, if~$y$ is a probability density function such as in \citet{radford2018improving}, this would mean that $\{\x^{i}\}_{i\in A}$ and $\{\x^{i}\}_{i\in B}$ are statistically independent.
The higher $sep(y;A,B)$ is, the farther~$y$ is from this situation, \ie~the more it models dependency between $\{\x^{i}\}_{i\in A}$ and $\{\x^{i}\}_{i\in B}$, or equivalently, the stronger the correlation it induces between the inputs indexed by~$A$ and those indexed by~$B$. 

\ifdefined\SQUEEZE \vspace{-1mm} \fi
\tocless\section{The Vocabulary Bottleneck \label{sec:The_Low_Rank_Bottleneck}}
\ifdefined\SQUEEZE \vspace{0mm} \fi

In this section, we theoretically establish the vocabulary bottleneck phenomenon in Transformer architectures, and its effect on the depth-to-width interplay. Specifically, we prove an upper bound on the separation rank of the analyzed Transformer architecture that grows exponentially with depth $L$ times \textit{the minimum} between the network width $d_x$ and the embedding rank $r$, and show it is tight for $L>\log_3d_x$. 

Without considering the embedding rank bottleneck, \citet{levine2020limits} have shown that for $L>\log_3d_x$, both depth and width contribute exponentially to the separation rank, and have provided extensive empirical corroboration of this prediction for a ``depth-\textbf{in}efficiency" regime of self-attention.  For the complementary regime of $L<\log_3d_x$ they prove a ``depth-efficiency" result, by which deepening is favorable over widening.
Our results imply that when the embedding rank is lower than the network width, the width related parameters are underutilized and ``depth-efficiency" kicks in immediately, even within the more practical $L>\log_3d_x$ regime. We formalize these notions below, and validate them empirically in the next section.   

The following theorems states that the network's capacity to model dependencies is harmed by a low rank embedding:
\ifdefined\SQUEEZE \vspace{-1mm} \fi
\begin{theorem}\label{theorem:embedding_rank_bottleneck_upper_bound}
	(upper bound on the separation rank) Let $y^{i,L, d_x,H,r}_p$ be the scalar function computing the $p^{\textrm{th}}$ entry of an output vector at position $i\in[N]$ of the $H$-headed depth-$L$ width-$d_x$ Transformer network\footnote{For simplicity, in this theorem we ignored the positional embedding dependencies and defer the full theorem to the appendix.} defined in eq.~\eqref{eq:our_layer}, where the embedding rank $r$ is defined by eq.~\eqref{eq:vocab_r} (vocabulary embedding) or eq.~\eqref{eq:conv_r} (convolution embedding). 
	Let $sep\left(y^{i,L, d_x,H,r}_p\right)$ denote its separation rank (section~\ref{sec:separation_rank}). 
	Then the following holds:
	\ifdefined\SQUEEZE \vspace{-0.5mm} \fi
	\begin{equation}\label{eq:upper_bound}
		\log\left(sep(y^{i,L, d_x,H, r}_p)\right)=
		\tilde{O}\left(L\cdot\min\{r,d_x\}\right)
	\end{equation}	
\end{theorem}
\ifdefined\SQUEEZE \vspace{-2mm} \fi

The theorem below simply states that under additional assumptions, the upper bound in the above theorem is asymptotically tight:
\begin{theorem}\label{theorem:embedding_rank_bottleneck_lower_bound}
	(lower bound on the separation rank) For $y^{i,L, d_x,H,r}_p$ as defined in in theorem~\ref{theorem:embedding_rank_bottleneck_upper_bound}, assume that $L>\log_{3}d_x$, $H<r$. Furthermore, for the vocabulary embedding case, assume that $N\rightarrow\infty$. Then for all values of the network weights but a set of Lebesgue measure zero, the following holds:
	\ifdefined\SQUEEZE \vspace{-1mm} \fi
	\begin{equation}\label{eq:lower_bound}
		\log\left(sep(y^{i,L, d_x,H, r}_p)\right)=
		\tilde{\Omega}\left(L\cdot\left(\min\{r,d_x\}-H\right)\right)
	\end{equation}
\end{theorem}

Note that the architectural assumptions that are added for the lower bound are practically reasonable. (1) $L>\log_{3}d_x$: for typical width $d_x$ of $1000$s, the log implies that the bound holds for networks of practical depths $\sim8$ and above; (2) $H<r$: the number of attention heads per layer $H$ is typically in the order of $10$s, while the width and rank are typically $100$s and above.

Note that while the $N\rightarrow\infty$ assumption in the vocabulary embedding are not practically reasonable, Our proof usage of $N$ is clearly wastefully, and we conjecture the lower bound holds for $N=\Omega\left(\nicefrac{r\cdot L}{\log_{3} r}\right)$ (see the appendix for detailed discussion). Moreover, ~\cite{bhojanapalli2020low} showed that $d_x<N$ limits self-attention expressivity, thus it is unlikely that huge $N$ contribute significantly to the network's capacity to model dependencies.

In the following, we outline the proof sketch for theorem~\ref{theorem:embedding_rank_bottleneck_upper_bound} (section~\ref{sec:The_Low_Rank_Bottleneck:sketch}) and then discuss three practical implications of the established vocabulary rank bottleneck (section~\ref{sec:The_Low_Rank_Bottleneck:implications}).

\ifdefined\SQUEEZE \vspace{-3mm} \fi
\tocless\subsection{Proof sketch for theorem~\ref{theorem:embedding_rank_bottleneck_upper_bound} \label{sec:The_Low_Rank_Bottleneck:sketch}}
\ifdefined\SQUEEZE \vspace{-1mm} \fi
We present below the proof outline for the upper bound in eq.~\ref{eq:upper_bound}, which establishes that the contribution of width to the separation rank is bottlenecked by the input embedding rank. We defer the full proof to the appendix, along with the proof of the lower bound in theorem.~\ref{theorem:embedding_rank_bottleneck_lower_bound}, which establishes the tightness of the upper bound.

First, notice that each self-attention layer, as defined in eq.~\eqref{eq:our_layer}, is a degree $3$ polynomial over its $N\cdot d_x$ scalar inputs. Since both the vocabulary embedding and the convolution embedding are linear mappings (eqs.~\eqref{eq:vocab} and~\eqref{eq:conv}), in both cases the whole network is a composition of $L$ degree-$3$ polynomials. Therefore $y^{i,L, d_x,H,r}_p$ is a degree $3^L$ polynomial over $N\cdot d_x$ variables. By definition, the separation rank of any monomial is $1$. In addition, the separation rank of sum of functions is upper bounded by the sum of theirs separation ranks. Thus, we upper bounded $sep(y^{i,L, d_x,H, r}_p)$ by the number of its monomials, which is at most $O\left(\left(3^L+N\cdot d_x\right)^{N\cdot d_x}\right)$ (a simple combinatorial bound detailed in the appendix).

The above analysis is agnostic to the first linear embedding layer. However, when $r<d_x$ this layer is important since the $N\cdot d_x$ variables have only $N\cdot r$ degrees of freedom: we can define a set of $N\cdot r$ variables %
which are a linear combinations of the original variables. %
Importantly $y^{i,L, d_x,H,r}_p$ is still a degree $3^L$ polynomial over the new variables, and each monomial of this polynomial has separation rank of $1$. By noticing that the summation over monomials is not tight, a more careful analysis, as done in the appendix, shows that for separation rank purposes, the effective degree of freedom is only $O\left(r\right)$, independent of $N$, thus concluding that $sep(y^{i,L, d_x,H, r}_p)$ is upper bounded by $O\left(\left(3^L+r\right)^{r}\right)$.

\ifdefined\SQUEEZE \vspace{-3mm} \fi
\tocless\subsection{Practical implications \label{sec:The_Low_Rank_Bottleneck:implications}}
\ifdefined\SQUEEZE \vspace{-1mm} \fi
Beyond quantifying the vocabulary bottleneck's effect on the network's ability to model dependencies (via separation rank), 
theorems~\ref{theorem:embedding_rank_bottleneck_upper_bound} and~\ref{theorem:embedding_rank_bottleneck_lower_bound} have direct implications on Transformer architecture design. We detail them in the following.

\ifdefined\SQUEEZE \vspace{-2mm} \fi
\tocless\subsubsection{The low rank expressivity bottleneck \label{sec:The_Low_Rank_Bottleneck:implications:performance}}
\ifdefined\SQUEEZE \vspace{-1mm} \fi
Since two functions can be equal only if they have the same separation rank, a network with a  low rank embedding $r<d_x$ cannot express the operation of a full rank $r= d_x$ network. 
As we demonstrate in section~\ref{sec:experiments:absoulte_degradation_exp} (figure~\ref{fig:low_rank_bottleneck}), this result translates into an empirical degradation in performance when $r< d_x$. 

A popular low vocabulary rank method is ALBERT, suggested in~\cite{Lan2020ALBERT}; we show in section~\ref{sec:experiments:absoulte_degradation_exp} that the low $\nicefrac{r}{d_x}=\nicefrac{128}{4096}$ ratio implemented in their network yields a $25\%$ redundancy in network size, \ie, a low rank network is surpassed by a full rank network with $75\%$ of the parameters. The above vocabulary bottleneck is even more common in non-linguistic domains. For example, \citet{rives2019biological} train a Transformer model with $\nicefrac{r}{d_x}=\nicefrac{33}{1280}$. The result in theorems~\ref{theorem:embedding_rank_bottleneck_upper_bound} and~\ref{theorem:embedding_rank_bottleneck_lower_bound} formalizes the sub-optimality of these settings.

\ifdefined\SQUEEZE \vspace{-2mm} \fi
\tocless\subsubsection{Effect on the depth-to-width interplay \label{sec:The_Low_Rank_Bottleneck:implications:depth_width}}
\ifdefined\SQUEEZE \vspace{-1mm} \fi
Beyond establishing a degradation in performance for low embedding rank Transformers, theorems~\ref{theorem:embedding_rank_bottleneck_upper_bound} and~\ref{theorem:embedding_rank_bottleneck_lower_bound} imply an advantage of deepening versus widening beyond the point of $d_x=r$, as deepening contributes exponentially more to the separation rank in this case. As we demonstrate in section~\ref{sec:experiments:implications_to_depth_width_interplay} (figures~\ref{fig:text_vocab_257_vs_2000} and~\ref{fig:depth_24_vs_48_albert_low_rank}), when comparing two 
Transformer architectures of depths $L^{\textrm{shallow}}<L^{\textrm{deep}}$ with the same embedding rank $r$ and the same number of parameters, the two networks perform comparably when  $d_x^{\textrm{shallow}}\leq r$, and the deeper network is better when $d_x^{\textrm{shallow}}> r$. 

This implication directly explains the observed depth-to-width ratio differences between language models and vision models. The Sparse Transformer~\cite{child2019sparsetransformer} over images, which is $128$ layers deep at the  same parameter count of the $12$-layered BERT-Base in language, has a small pixel-intensity vocabulary (each pixel is translated into $3$ input tokens corresponding to its color channels), which caps the contribution of width. The same vocabulary is used in the ablation of~\citet{henighan2020scaling}, which attributes the difference in optimal depth-to-width ratio to the difference in data modalities. The result in theorems~\ref{theorem:embedding_rank_bottleneck_upper_bound} and~\ref{theorem:embedding_rank_bottleneck_lower_bound}, along with the corroborating experiments in section~\ref{sec:experiments:implications_to_depth_width_interplay}, implies that this phenomenon is \textit{modality independent}, and is in fact related to architecture expressivity.

\ifdefined\SQUEEZE \vspace{-2mm} \fi
\tocless\subsubsection{A mid-architecture bottleneck --  width caps the internal attention dimension \label{sec:The_Low_Rank_Bottleneck:implications:t5}}
\ifdefined\SQUEEZE \vspace{-1mm} \fi
The above implications relate to the vocabulary bottleneck caused by a low input embedding rank. 
By observing that the upper bound on the separation rank does not depend on the number of attention heads $H$, 
we establish a mid-architecture bottleneck related to this architectural aspect, which affects leading Transformer architectures.

Specifically, following the original implementation of ~\citet{vaswani2017transformer}, most common Transformer architectures set the number of heads to be $H=\nicefrac{d_x}{d_a}$, \ie, the network width divided by the internal attention dimension (see text below eq.~\eqref{eq:our_layer}). 
However, in their effort to drastically increase network size under hardware restrictions,~\citet{raffel2019exploring} trained their unprecedentedly-sized~$11$B parameter T5 model by decoupling the above:
They train a width $d_x=1$K network but increase the number of attention heads per layer to $H=128$ while keeping the attention dimension fixed at $d_a=128$. Thus, T5-11B achieves an internal attention embedding of $ H\cdot d_a=16$K before projecting back down to the width value of $1$K between layers.	

However, since the bounds in theorems~\ref{theorem:embedding_rank_bottleneck_upper_bound} and~\ref{theorem:embedding_rank_bottleneck_lower_bound} are capped at $d_x$, our theoretical framework predicts that this form of parameter increase is suboptimal, \ie, adding parameters to the internal attention representation beyond $H\cdot d_a=d_x$ is inferior to simply increasing the width $d_x$. 
We verify this conclusion in section~\ref{sec:experiments:t5_bottleneck} (figure~\ref{fig:t5_expansion}), where we establish that the architectural configuration presented in T5 is indeed highly parameter redundant.

\ifdefined\SQUEEZE \vspace{-2mm} \fi
\tocless\section{Experiments \label{sec:experiments}}
\ifdefined\SQUEEZE \vspace{-1mm} \fi
In the previous sections, we analyzed a simplified version of Transformer networks (described in  section~\ref{sec:TransformerArchitecture}). For this class, we proved the existence of a low-rank embedding bottleneck that limits the contribution of network width to the Transformer expressivity.
In this section, we demonstrate that \textbf{our theoretical predictions are manifested in common Transformer networks}; the experiments below were conducted over common architectures which include all operations that were omitted in our theoretical analysis.

Since unidirectional models are more stable to vocabulary variations than bidirectional models \cite{levine2021pmimasking}, we trained decoder-only language models, by optimizing the autoregressive log-likelihood of the training examples for $1M$ steps.
We perform our experiments over language in order to establish that architecture expressivity causes the variation in the optimal depth-to-width ratio: we demonstrate that the vocabulary bottleneck causes the examined language modality to exhibit the same trends that were attributed to other modalities    by~\citet{henighan2020scaling}.  

Our training set was English Wikipedia, BookCorpus and OpenWebText, with a total size of $60$G. 
We report the loss on a held out test set of~$40$K sequences. 
Notably, we estimated the variance of the pretraining and evaluation procedure by rerunning $7$ of the trained architectures five times each, and found it to be very low -- the reported test loss is stable up to $~10^{-3}$.
The remainder of the training details are given in the appendix.

\ifdefined\SQUEEZE \vspace{-2mm} \fi 
\tocless\subsection{Rank bottleneck degrades performance \label{sec:experiments:absoulte_degradation_exp}}
\ifdefined\SQUEEZE \vspace{-1mm} \fi
\begin{figure}[t]
	\vskip 0.2in
	\begin{center}
		\centerline{\includegraphics[width=\columnwidth]{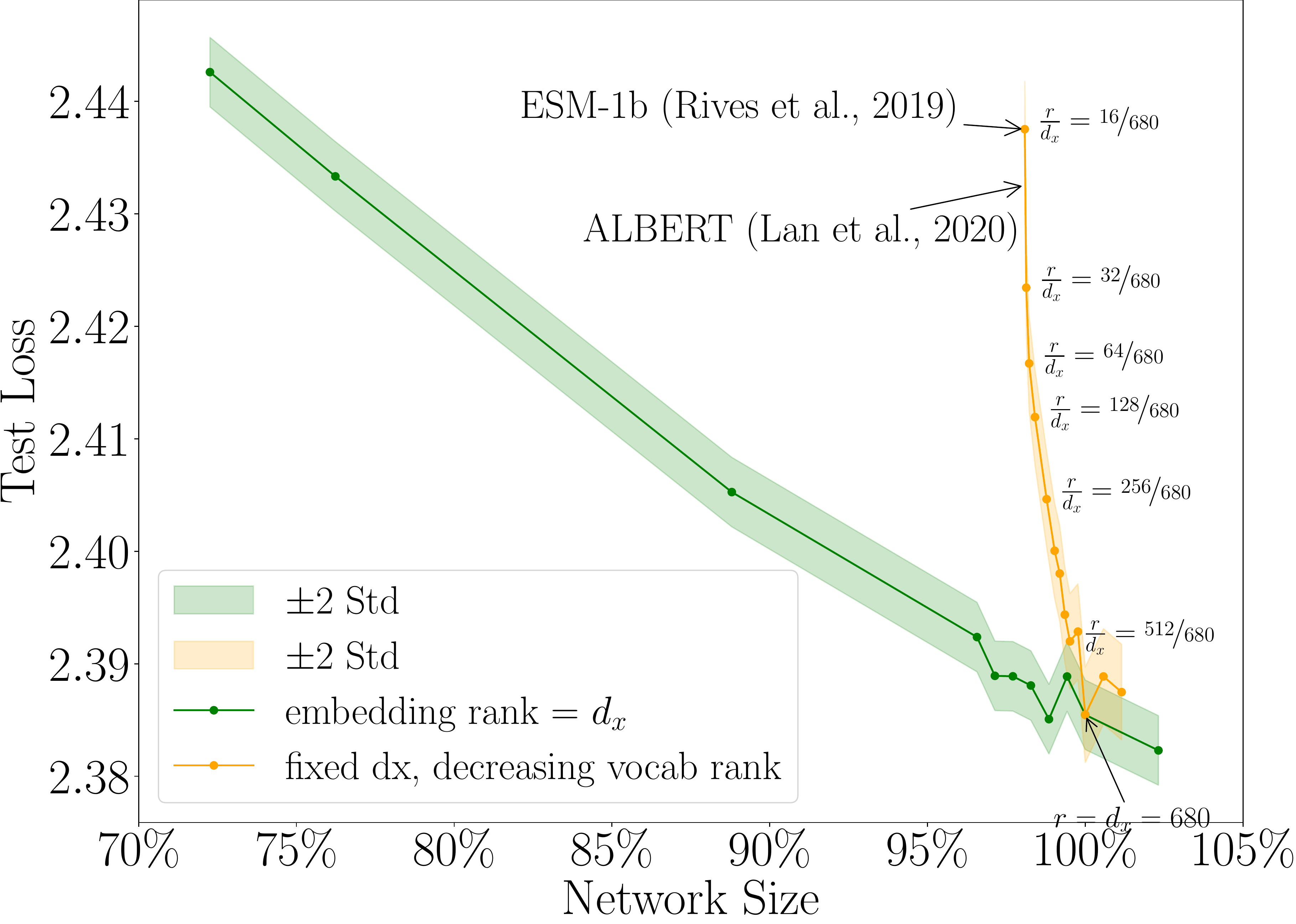}}
		\caption{An experimental validation of the low-rank embedding bottleneck. We reduce the network size by decreasing either the width of all layers (green), or by only decreasing the input embedding rank (orange). A rapid degradation is observed when the input embedding rank is decreased. This degradation affects leading models in NLP (ALBERT) and in other domains (ESM-1b).
		}
		\label{fig:low_rank_bottleneck}
		\ifdefined\SQUEEZE \vspace{-2mm} \fi
	\end{center}
	\vskip -0.2in
	\ifdefined\SQUEEZE \vspace{-4mm} \fi
\end{figure}
Theorems~\ref{theorem:embedding_rank_bottleneck_upper_bound} and~\ref{theorem:embedding_rank_bottleneck_lower_bound} reveal a low-rank embedding-bottleneck phenomenon in Transformer architectures. In this subsection, we demonstrate this phenomenon by factoring the input embedding matrix into two matrices of dimensions $d_x \times r$, $r \times V$, while incrementally decreasing the rank $r$. This is similar to the approach suggested in ALBERT \cite{Lan2020ALBERT} for reducing the parameter count. 
Specifically, we trained depth-$12$ networks with embedding ranks that range between $16$ and the full rank of $d_x=680$. Additionally, we trained a baseline of full input embedding rank depth-$12$ networks of sizes varying between $50$M and $75$M parameters (full details on the widths are given in the appendix).

In common language model implementations, the input embedding weights are shared with the final classification weights (referred to as weight tying~\cite{press-wolf-2017-using}). In our experiment, we wanted ensure that the performance degradation is caused by the embedding bottleneck and not by the softmax bottleneck that \citet{yang2018breaking} establish regarding the final classification operation. Therefore, we did not perform the above weight tying, and kept the classification weights matrix a fully ranked $V \times d_x$ matrix, even when we decreased the rank of the input embedding matrix.

Figure~\ref{fig:low_rank_bottleneck} shows that when decreasing the input embedding rank, the loss increases much more rapidly than when reducing the network size by decreasing the width $d_x$. A network with $\sim70$ million parameters and an ``embedding-rank to width" ratio of $\nicefrac{r}{d_x}=\nicefrac{16}{680}$, is comparable in performance to a non-bottlenecked network that is $25\%$ smaller. This low-rank ratio is close the the ratio of ALBERT-xxlarge $\nicefrac{r}{d_x}=\nicefrac{128}{4096}$ as well as to the ratio caused by the $V=33$ and $d_x=1280$ of the leading $650\text{M}$ parameter ESM-1b \cite{rives2019biological} protein model, implying that their network could have been strengthened by constructing a larger vocabulary (perhaps a ``word-level" equivalent).

\ifdefined\SQUEEZE \vspace{-2mm} \fi
\tocless\subsection{Vocabulary affects the depth-to-width interplay \label{sec:experiments:implications_to_depth_width_interplay}}
\ifdefined\SQUEEZE \vspace{-0mm} \fi
The results of the previous section directly imply that by limiting the embedding rank, a small vocabulary can harm performance. 
In this section we verify the second conclusion of theorems~\ref{theorem:embedding_rank_bottleneck_upper_bound} and~\ref{theorem:embedding_rank_bottleneck_lower_bound}, which states that for either $V<d_x$ or $r<d_x$, it is better to use narrower and deeper models. 

\tocless\subsubsection{Small vocabulary size $\mathbf{V<d_x}$}
We compared networks of depths $L^{\textrm{shallow}}=24,L^{\textrm{deep}}=48$, of sizes ranging between $5$M and $110$M, when three different BPE \cite{sennrich-etal-2016-neural} vocabularies of sizes $V=257, 500, 2000$ are used (full details on the vocabularies and trained architectures are given in the appendix).
Figures~\ref{fig:text_vocab_257_vs_2000}(a)-(c) show a clear trend: the smaller the vocabulary, the sooner the deeper networks begin outperforming the shallow ones. 
Moreover, at the ``depth-efficiency" point, for which the deeper network starts outperforming the shallow one, the width of the shallower network is around the vocabulary size. 

In addition, we show in the appendix that this does not occur when the vocabulary size exceeds the network width and does not constitute a bottleneck. In this case, the vocabulary size has negligible effect on the ``depth-efficiency" point. For example, we show that the ``depth-efficiency" point of  GPT-2's vocabulary \cite{radford2019language} remains very close to that of our $V=2000$ vocabulary, even though the GPT-2 vocabulary is $\sim25X$ larger. This is directly in line with our prediction when there is no  vocabulary bottleneck. 

Overall, figure~\ref{fig:text_vocab_257_vs_2000} clearly shows that the same phenomenon that occurred over images and mathematical data, of tilting the depth-to-width ratio towards depth, occurs also in the case of language when the vocabulary bottleneck is present.
This phenomenon was attributed by~\citet{henighan2020scaling} to the difference in data modalities, but the above outcomes reinforce the network expressivity related interpretation for its origin.

\begin{figure}[t]
	\vskip 0.2in
	\begin{center}
		\centerline{\includegraphics[width=0.98\columnwidth]{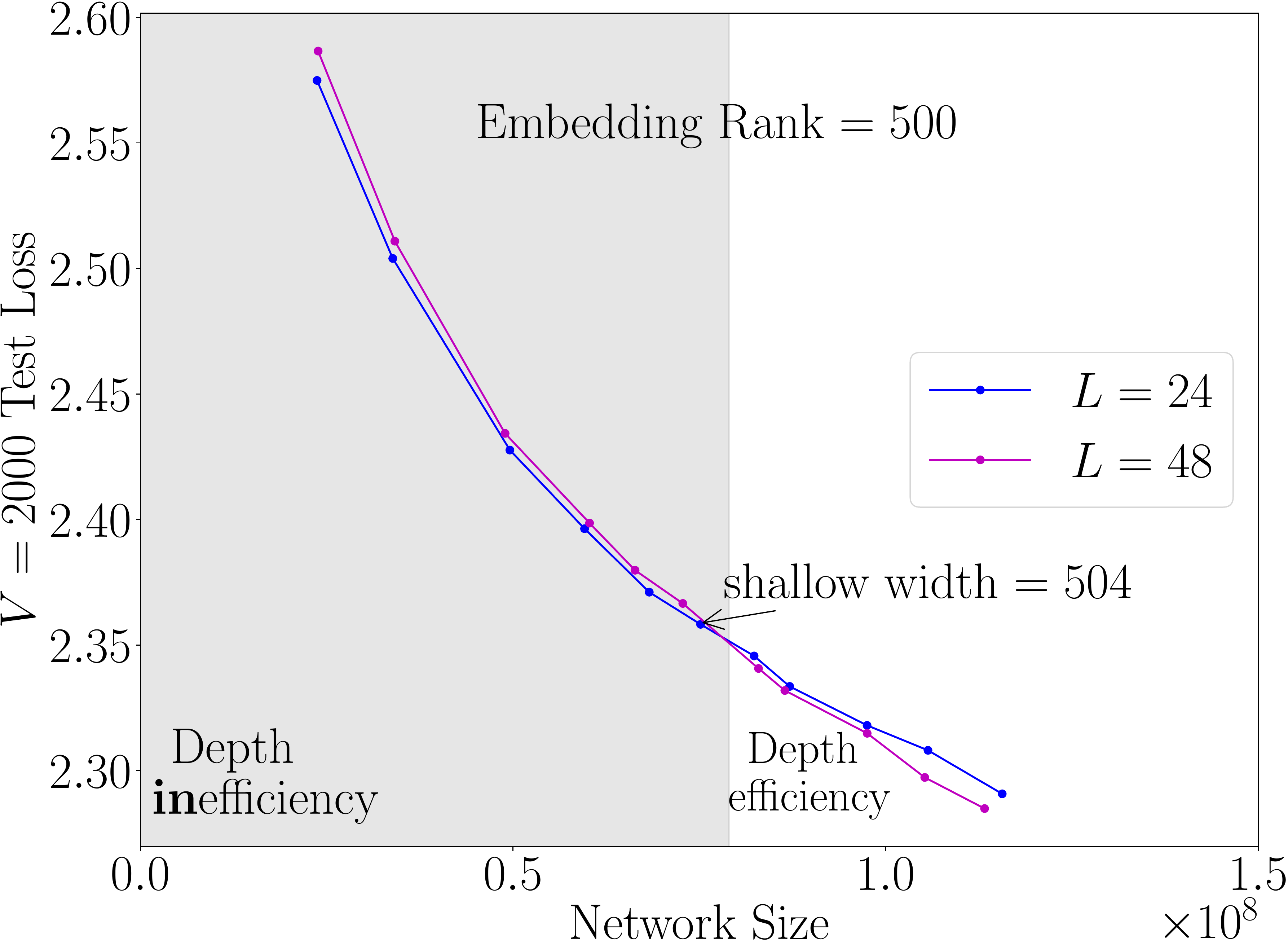}}
		\caption{An experiment with the same $V=2000$ vocabulary of figure~\ref{fig:text_vocab_257_vs_2000}(c), when the rank of the input embedding matrix is restricted to $500$.
			Deeper networks gain an advantage over shallower ones due to the rank bottleneck.
		}
		\label{fig:depth_24_vs_48_albert_low_rank}
		\ifdefined\SQUEEZE \vspace{-0mm} \fi
	\end{center}
	\vskip -0.2in
	\ifdefined\SQUEEZE \vspace{-0mm} \fi
\end{figure}

\tocless\subsubsection{Low input rank $\mathbf{r<d_x}$}

Notably, since the input sequence length $N$ is fixed,  using a smaller vocabulary implies that the model sees less text. 
Therefore, the above differences in the depth-to-width trade-offs could be attributed to the variation in training data rather than to the effect of the rank bottleneck. Figure~\ref{fig:depth_24_vs_48_albert_low_rank} establishes the rank bottleneck as the influencing factor by comparing $L^{\textrm{shallow}}=24,L^{\textrm{deep}}=48$ networks with the same vocabulary size of $2000$, while limiting the embedding rank to $500$. Similarly to the vocabulary varying experiments in  figure~\ref{fig:text_vocab_257_vs_2000}, the deeper network surpassed the shallower one when the latter's width reached its rank. 

Interestingly, in this case the transition is very close to the rank of $500$, at $d_x^{\textrm{shallow}}=504$, while in the $V=500$ setting of figure~\ref{fig:text_vocab_257_vs_2000}b the transition occurs earlier, around $d_x^{\textrm{shallow}}=408$. We conjecture that when the vocabulary itself is small, not all tokens are equally utilized, whereas for a low rank vocabulary matrix with a larger vocabulary, a better utilization of the rank is achieved. Either way, we see that the advantage of depth is indeed an implication of the embedding rank bottleneck.

\ifdefined\SQUEEZE \vspace{-2mm} \fi
\tocless\subsection{Width bottlenecks the attention dimension \label{sec:experiments:t5_bottleneck}}
\ifdefined\SQUEEZE \vspace{-1mm} \fi
In this section, we demonstrate the effect of the mid-architecture bottleneck identified in section~\ref{sec:The_Low_Rank_Bottleneck:implications:t5},
and show that performance is degraded when the internal attention representation dimension exceeds network width, \ie, when $H\cdot d_a>d_x$.

We trained networks of sizes ranging between $30$M and $70$M, while varying the bottleneck ratio of $\nicefrac{H\cdot d_a}{d_x}$ within $\{1,2,4,8,16\}$, where $1$ is the baseline value. For each bottleneck ratio and network size, we chose the best depth in $\left\{12,18,24\right\}$. We provide the results of all the depths per bottleneck ratio in the appendix, showing that the performance difference between values of the bottleneck ratio is much larger than the variation between different depths per bottleneck ratio.

Figure~\ref{fig:t5_expansion} shows that by fixing $d_a$ and increasing $H$, performance is indeed degraded for  $d_x< H\cdot d_a$. Importantly, the degradation is monotone with the $\nicefrac{H\cdot d_a}{d_x}$ bottleneck ratio. 
The second largest T5 model of~\citet{raffel2019exploring} has distributed its $3$B parameters by using width of $d_x=1$K and a bottleneck ratio of $\nicefrac{H\cdot d_a}{d_x}=4$, corresponding to the yellow line in figure~\ref{fig:t5_expansion}. The figure shows that the performance of a network with this ratio can be achieved by a baseline network with $\sim75\%$ of the parameters (green). 
The largest T5 model has a width of $d_x=1$K and a bottleneck ratio $\nicefrac{H\cdot d_a}{d_x}=16$, corresponding to the red line in figure~\ref{fig:t5_expansion}. 
The performance of a network with this ratio can be achieved by a baseline network with $\sim55\%$ of the parameters, implying that T5-11B could have been trained with $\sim6$B parameters with no degradation. 

It is noteworthy that T5.1.1 XL and T5.1.1 XXL~\cite{craffel2021t511, xue-etal-2021-mt5,xue2021byt5}, more recent, modified, versions of T5-3B and T5-11B, have increased widths of 2K and 4K, respectively, with corresponding number of heads such that $\nicefrac{H\cdot d_a}{d_x}=1$. 
Our theoretical analysis and its empirical corroboration in figure~\ref{fig:t5_expansion} highlight the importance of this architectural aspect.

\begin{figure}[t]
	\vskip 0.2in
	\begin{center}
		\centerline{\includegraphics[width=\columnwidth]{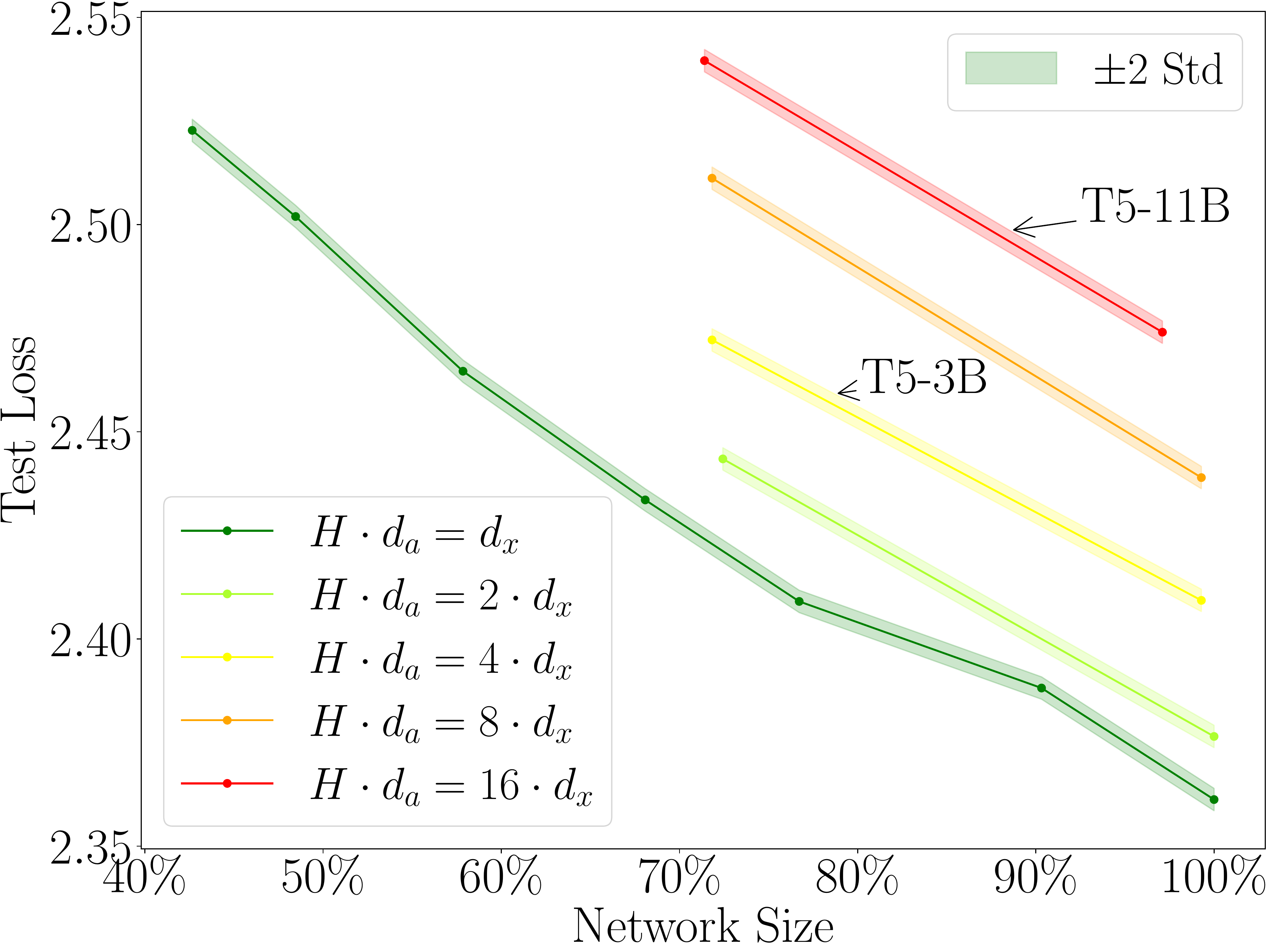}}
		\caption{The $d_x< H\cdot d_a$ bottleneck of the T5 architecture degrades performance.  As the bottleneck ratio increases, smaller baseline architectures outperform variants that invest parameters in the internal attention representation. T5-11B was trained with $H\cdot d_a = 16d_x$, implying a $\sim45\%$ parameter redundancy.}
		\label{fig:t5_expansion}
		\ifdefined\SQUEEZE \vspace{-2mm} \fi
	\end{center}
	\vskip -0.2in
	\ifdefined\SQUEEZE \vspace{-3mm} \fi
\end{figure}

\ifdefined\SQUEEZE \vspace{-3mm} \fi
\tocless\section{Discussion \label{sec:discussion}}
\ifdefined\SQUEEZE \vspace{-1mm} \fi
After observing a variation in the optimal depth-to-width ratios of Transformers when applied to different domains, previous works have concluded that this architectural design aspect depends on the input data modality.
Our theoretical framework, reinforced by supporting experiments, indicates the variation in common vocabulary sizes across domains as the root of the observed depth-to-width variability. 
Specifically, we prove and validate empirically that when the vocabulary size or rank is smaller than the network width, the contribution of depth to expressivity is boosted and the optimal depth-to-width ratio per Transformer architecture size increases.

Our results provide practical domain independent guidelines for Transformer architecture design.
If possible, it is good to increase the input embedding rank such that it surpasses the network width. 
For example, the largest Vision Transformer, ViT-Huge~\cite{dosovitskiy2021an}, has width of $1280$ but an input rank of $768$, which constitute a bottleneck. This can easily be alleviated if noticed, for example via an additional convolutional layer over the inputs (see eq.~\eqref{eq:conv_r}).
Alternatively, the network depth can be increased at the expense of its width, as previous works propose (without noting the input rank origins).

However, it is important to note the vocabulary independent depth-to-width Transformer expressivity considerations given in~\cite{levine2020limits}: for a given parameter budget, a network can be too deep.
More generally, from either expressivity, optimization, or even engineering considerations, a large width is beneficial (\eg, allows for a more effective parallelization). 
Therefore, methods for increasing the vocabulary rank, such as an elaborate convolutional embedding, or alternatively, different coding methods, should be considered for training a large width Transformer that does not suffer from the identified vocabulary bottleneck.

Regarding impact on NLP (still the major consumer of Transformers),  
while the vocabulary size is commonly larger than network width in this field, and does not therefore constitute a bottleneck, \citet{levine2020limits} predict that scaling up to $1$-Trillion parameter networks and beyond, would require massive widening of $d_x=30$K and more. This reaches standard vocabulary sizes, so the vocabulary bottleneck should be noted. 
Moreover, \citet{xue2021byt5} concurrently demonstrate that a byte-level vocabulary can work well in Transformer-based LMs. 
As we show, when proposing architectural modifications or vocabulary size reductions, our theoretically established bottlenecks should be considered, as they translate into practically manifested redundancies (see figures~\ref{fig:low_rank_bottleneck} and~\ref{fig:t5_expansion}). Overall, our work aims to provide timely theoretical interpretations, to help guide the rapid empirical advances of our field.

\section*{Acknowledgments}
This research was supported by the ERC (European Research Council) and the ISF (Israel Science Foundation). Experiments were performed with Cloud TPUs and supported by Google's TPU Research Cloud (TRC). Experiments results management done using Neptune and supported by Neptune.ai's research program.
Yoav Levine was supported by the Israel Academy of Sciences Adams fellowship.

%% file: appendix_body.tex
	\section{Upper bounds on the separation rank}
In the following section, we show how an upper bound on the separation rank is implied by the rank of embedding. 
\subsection{Preliminaries}
We will use the notation of $\multiset{n}{k}$ -- the multiset coefficient, given in the binomial form by
$\binom{n+k-1}{k}$. We will use the identity $\abs{\left\{a_1\ldots a_n\in\mathbb{Z}\geq0:\sum_{r=1}^{n}a_{r}=k\right\} } = \multiset{n}{k}$.
In addition, we will use the following two lemmas from \cite{levine2020limits} regarding the composition of $L$ self-attention layers, and inequality of arithmetic and geometric multiset coefficient means.
\begin{lemma}\label{lemma:self_attention_compositions}
	Defining $C\left(L\right)\coloneqq\frac{3^{L}-1}{2}$, any depth $L$
	composition of self-attention layers defined in eq.~\ifdefined\NOBODY 5\else\ref{eq:our_layer}\fi ~of the main text can be written as:
	\begin{align}\label{eq:gl_explicit_form}
y^{i,L,d_{x},H}\left(y^{0,1},...,y^{0,N}\right)=\sum_{j_{1},\dots,j_{C\left(L\right)}=1}^{N}\sum_{h\in\left[H\right]^{\left[C\left(L\right)\right]}}\sum_{r_{1},\dots,r_{C\left(L\right)+1}=1}^{d_{a}}B_{r_{1},p}^{\left(0,h\right)}\left(\prod_{c=1}^{C\left(L\right)+1}\left\langle A_{r_{c}}^{\left(c,h\right)},\y^{0,j_{c}}\right\rangle \right)\left(\prod_{c=1}^{C\left(L\right)}\left\langle B_{r_{c+1}}^{\left(c,h\right)},\y^{0,j_{c}}\right\rangle \right)
	\end{align}	
	Where $\forall h\in\left[H\right]^{\left[C\left(L\right)\right]}\,1\leq c\leq C\left(L\right)+1\quad A^{\left(c,h\right)},B^{\left(c,h\right)}\in\R^{d_{a}\times d_{x}}$ and for convenient $j_{C\left(L\right)+1}\coloneqq i$.
\end{lemma}
\begin{lemma}
	\label{lemma:means_inequality}	
	Let $n,k\in\mathbb{N}$ and $\phi:\mathbb{N}^{k}\rightarrow\mathbb{N}\coloneqq r_{1},\dots r_{k}\vdash\prod_{j=1}^{k}\left(\binom{n}{r_{j}}\right)$
	then:
	\[
	\forall r_{z},\dots r_{k}\in\mathbb{N}\quad\phi\left(r_{1},\dots r_{k}\right)\leq\frac{\left(\prod_{t=1}^{n-1}\left(\frac{M}{k}+t\right)\right)^{k}}{\left(\left(n-1\right)!\right)^{k}}
	\]
	where $M\coloneqq\sum_{j=1}^{k}r_{j}$
\end{lemma}
Finally, we will use the following lemma to upper bound the multiset coefficient:
\begin{lemma}\label{lemma:multiset_upper_bound}
	$\ensuremath{\left(\binom{n}{k}\right)\leq\left(\frac{2e\left(n+k\right)}{n}\right)}^{n}$
\end{lemma}
\begin{proof}
	: by using the inequality $\binom{n}{k}\leq\left(\frac{en}{k}\right)^{k}$ we have$$\left(\binom{n}{k}\right)=\binom{n+k-1}{n-1}\leq\left(\frac{2e\left(n+k\right)}{n}\right)^{n}$$
\end{proof}

\subsection{Vocabulary based embedding}\label{sec:vocab_uuper_bound}
In the following theorem, we show how an upper bound on the separation rank is implied by the rank of vocabulary matrix.
\begin{theorem}\label{theorem:vocab_uuper_bound}
	Let $y^{i,L, d_x,H,r}_p$ be the scalar function computing the $p^{\textrm{th}}$ entry of an output vector at position $i\in[N]$ of the $H$-headed depth-$L$ width-$d_x$ Transformer network defined in eqs.~\ifdefined\NOBODY 1\else\ref{eq:vocab}\fi ~and ~\ifdefined\NOBODY 5\else\ref{eq:our_layer}\fi ~of the main text, where the embedding rank $r$ is defined by eq.~\ifdefined\NOBODY 3\else\ref{eq:vocab_r}\fi ~of the main text.
	Let $r_{e}$ denote the rank of the positional embedding matrix and $sep\left(y^{i,L, d_x,H,r}_p\right)$ denote its separation rank \wrt ~any partition $P\cupdot Q=\left[N\right]$. 
	Then the following holds:
	\begin{equation}\label{eq:vocab_uuper_bound}
		sep(y^{i,L, d_x,H, r}_p)\le\multiset{r+r_e}{\cdot 3^L}\multiset{4}{3^L}\left(3^L+1\right)^{r+r_e}
	\end{equation}
\end{theorem}
\begin{proof}
	By the embedding low-rank assumptions, there exists  $M^{\text{vocab}}\in\R^{r\times V},M^{\text{pos}}\in\R^{r_e\times N}$ and $M^{\text{low-rank}}\in\R^{d_x\times r},P^{\text{low-rank}}\in\R^{d_x\times r_e}$ such that
	\begin{align}\label{eq:low_rank_vocab}
		\y^{0,i}=M^{\text{low-rank}}M^{\text{vocab}} ~\hat{\mathbf{w}}^i+P^{\text{low-rank}}M^{\text{pos}}_i
	\end{align}
	So we begin by substituting $\y^{0,i}$ in eq~\ref{eq:gl_explicit_form} (for convenience, we denote $j_{C\left(L\right)+1}\coloneqq i$):
	\begin{align*}
y^{i,L,d_{x},H,r}\left(w^{0},...,w^{N}\right)&=\sum_{j_{1},\dots,j_{C\left(L\right)}=1}^{N}\sum_{h\in\left[H\right]^{\left[C\left(L\right)\right]}}\sum_{r_{1},\dots,r_{C\left(L\right)+1}=1}^{d_{a}}B_{r_{1},p}^{\left(0,h\right)}\left(\prod_{c=1}^{C\left(L\right)+1}\left\langle A_{r_{c}}^{\left(c,h\right)},M^{\text{low-rank}}M^{\text{vocab}}~\hat{\mathbf{w}}^{j_{c}}+P^{\text{low-rank}}M_{j_{c}}^{\text{pos}}\right\rangle \right)\\&~~~~~~~~~~~~~~~~~~~~~~~~~~~~~~~~~~~~~~~~~~~~~~~~~~~~~~~~~~~~~~~~~~~~~~~~~~~~~\left(~\prod_{c=1}^{C\left(L\right)}~~\left\langle B_{r_{c+1}}^{\left(c,h\right)},M^{\text{low-rank}}M^{\text{vocab}}~\hat{\mathbf{w}}^{j_{c}}+P^{\text{low-rank}}M_{j_{c}}^{\text{pos}}\right\rangle \right)
	\end{align*}
	And separating between the tokens and the positional embeddings:
	\begin{align*}
=\underbrace{\sum_{\substack{\substack{I_{A}\subseteq\left[C\left(L\right)+1\right]\\
				\substack{I_{B}\subseteq\left[C\left(L\right)\right]}
			}
		}
}}_{\text{the indices of tokens}}\sum_{j_{1},\dots,j_{C\left(L\right)}=1}^{N}\sum_{h\in\left[H\right]^{\left[C\left(L\right)\right]}}\sum_{r_{1},\dots,r_{C\left(L\right)+1}=1}^{d_{a}}\underbrace{\left(\prod_{c\in\left[C\left(L\right)+1\right]\setminus I_{A}}\left\langle A_{r_{c}}^{\left(c,h\right)},P^{\text{low-rank}}M_{j_{c}}^{\text{pos}}\right\rangle \right)\left(\prod_{c\in\left[C\left(L\right)\right]\setminus I_{B}}\left\langle B_{r_{c+1}}^{\left(c,h\right)},P^{\text{low-rank}}M_{j_{c}}^{\text{pos}}\right\rangle \right)}_{\text{The positional embeddings}}\\B_{r_{1},p}^{\left(0,h\right)}\underbrace{\left(\prod_{c\in I_{A}}\left\langle A_{r_{c}}^{\left(c,h\right)},M^{\text{low-rank}}M^{\text{vocab}}~\hat{\mathbf{w}}^{j_{c}}\right\rangle \right)\left(\prod_{c\in I_{B}}\left\langle B_{r_{c+1}}^{\left(c,h\right)},M^{\text{low-rank}}M^{\text{vocab}}~\hat{\mathbf{w}}^{j_{c}}\right\rangle \right)}_{\text{The tokens}}
	\end{align*}
	Now we can open the inner products, explicitly writing the indices:
	\begin{align*}
=\sum_{\substack{\substack{I_{A}\subseteq\left[C\left(L\right)+1\right]\\
			\substack{I_{B}\subseteq\left[C\left(L\right)\right]}
		}
	}
}\sum_{\substack{\alpha_{1},\dots,\alpha_{C\left(L\right)+1}\\
		\beta_{1},\dots,\beta_{C\left(L\right)}
	}
	=1}^{r}\sum_{\substack{\sigma_{1},\dots,\sigma_{C\left(L\right)+1}\\
		\mu_{1},\dots,\mu_{C\left(L\right)}
	}
	=1}^{r_{e}}\sum_{j_{1},\dots,j_{C\left(L\right)}=1}^{N}\sum_{h\in\left[H\right]^{\left[C\left(L\right)\right]}}\sum_{r_{1},\dots,r_{C\left(L\right)+1}=1}^{d_{a}}B_{r_{1},p}^{\left(0,h\right)}\\\left(\sum_{\substack{\gamma_{1},\dots,\gamma_{C\left(L\right)+1}\\
		\delta_{1},\dots,\delta_{C\left(L\right)}
	}
	=1}^{d_{x}}\left(\prod_{c\in\left[C\left(L\right)+1\right]\setminus I_{A}}A_{r_{c},\gamma_{c}}^{\left(c,h\right)}P_{\gamma_{c},\sigma_{c}}^{\text{low-rank}}M_{\sigma_{c},j_{c}}^{\text{pos}}\right)\left(\prod_{c\in\left[C\left(L\right)\right]\setminus I_{B}}B_{r_{c+1},\delta_{c}}^{\left(c,h\right)}P_{\delta_{c},\mu_{c}}^{\text{low-rank}}M_{\mu_{c},j_{c}}^{\text{pos}}\right)\right)\\\left(\sum_{\substack{\gamma_{1},\dots,\gamma_{C\left(L\right)+1}\\
		\delta_{1},\dots,\delta_{C\left(L\right)}
	}
	=1}^{d_{x}}\left(\prod_{c\in I_{A}}A_{r_{c},\gamma_{c}}^{\left(c,h\right)}M_{\gamma_{c},\alpha_{c}}^{\text{low-rank}}M_{\alpha_{c},w^{j_{c}}}^{\text{vocab}}\right)\left(\prod_{c\in I_{B}}B_{r_{c+1},\delta_{c}}^{\left(c,h\right)}M_{\delta_{c},\beta_{c}}^{\text{low-rank}}M_{\beta_{c},w^{j_{c}}}^{\text{vocab}}\right)\right)
	\end{align*}
	And separating between coefficients and $~w$’s:
	\begin{align*}
=\sum_{\substack{\substack{I_{A}\subseteq\left[C\left(L\right)+1\right]\\
			\substack{I_{B}\subseteq\left[C\left(L\right)\right]}
		}
	}
}\sum_{\substack{\alpha_{1},\dots,\alpha_{C\left(L\right)+1}\\
		\beta_{1},\dots,\beta_{C\left(L\right)}
	}
	=1}^{r}\sum_{\substack{\sigma_{1},\dots,\sigma_{C\left(L\right)+1}\\
		\mu_{1},\dots,\mu_{C\left(L\right)}
	}
	=1}^{r_{e}}\sum_{j_{1},\dots,j_{C\left(L\right)}=1}^{N}\tau_{\substack{I_{A},I_{B}}
	,\alpha_{1},\dots,\mu_{C\left(L\right)}}\left(\prod_{c\in\left[C\left(L\right)+1\right]\setminus I_{A}}M_{\sigma_{c},j_{c}}^{\text{pos}}\right)\left(\prod_{c\in\left[C\left(L\right)\right]\setminus I_{B}}M_{\mu_{c},j_{c}}^{\text{pos}}\right)\\\left(\prod_{c\in I_{A}}M_{\alpha_{c},w^{j_{c}}}^{\text{vocab}}\right)\left(\prod_{c\in I_{B}}M_{\beta_{c},w^{j_{c}}}^{\text{vocab}}\right)
	\end{align*}
	Where the coefficients are equals to:
	\begin{align*}
\tau_{\substack{I_{A},I_{B}}
	,\alpha_{1},\dots,\mu_{C\left(L\right)}}\coloneqq\sum_{h\in\left[H\right]^{\left[C\left(L\right)\right]}}\sum_{r_{1},\dots,r_{C\left(L\right)+1}=1}^{d_{a}}B_{r_{1},p}^{\left(0,h\right)}\left[\sum_{\substack{\gamma_{1},\dots,\gamma_{C\left(L\right)+1}\\
		\delta_{1},\dots,\delta_{C\left(L\right)}
	}
	=1}^{d_{x}}\left(\prod_{c\in I_{A}}A_{r_{c},\gamma_{c}}^{\left(c,h\right)}M_{\gamma_{c},\alpha_{c}}^{\text{low-rank}}\right)\left(\prod_{c\in I_{B}}B_{r_{c+1},\delta_{c}}^{\left(c,h\right)}M_{\delta_{c},\beta_{c}}^{\text{low-rank}}\right)\right]\\\left[\sum_{\substack{\gamma_{1},\dots,\gamma_{C\left(L\right)+1}\\
		\delta_{1},\dots,\delta_{C\left(L\right)}
	}
	=1}^{d_{x}}\left(\prod_{c\in\left[C\left(L\right)+1\right]\setminus I_{A}}A_{r_{c},\gamma_{c}}^{\left(c,h\right)}P_{\gamma_{c},\sigma_{c}}^{\text{low-rank}}\right)\left(\prod_{c\in\left[C\left(L\right)\right]\setminus I_{B}}B_{r_{c+1},\delta_{c}}^{\left(c,h\right)}P_{\delta_{c},\mu_{c}}^{\text{low-rank}}\right)\right]
	\end{align*}
	Now we can group monomials by the powers $n_{1},\dots,n_{r},p_{1},\dots,p_{r_{e}}$ of each coordinate: 
	\begin{align*}
=\underbrace{\sum_{N_{A\bigtriangleup B}=0}^{C\left(L\right)+1}\sum_{N_{A\cap B}=0}^{C\left(L\right)}}_{\text{\ensuremath{\substack{\text{How many }j_{c}\text{ indices}\\
				\text{are token indices}
			}
}}}\underbrace{\sum_{\substack{n_{1}+\dots+n_{r}=N_{A\bigtriangleup B}+2N_{A\cap B}\\
			p_{1}+\dots+p_{r_{e}}=2C\left(L\right)+1-N_{A\bigtriangleup B}-2N_{A\cap B}
		}
}}_{\text{The powers}}\underbrace{\sum_{\substack{m_{1}+\dots+m_{N}=N_{A\bigtriangleup B}+2N_{A\cap B}\\
			z_{1}+\dots+z_{N}=2C\left(L\right)+1-N_{A\bigtriangleup B}-2N_{A\cap B}\\
			\forall j\in\left[N\right]\,m_{j}+z_{j}\equiv\begin{dcases}
				\scriptstyle
				1\,\text{mod}\,2 & j=i\\
				\scriptstyle
				0\,\text{mod}\,2 & j\neq i
			\end{dcases}
		}
}}_{\substack{\text{\text{How many }\text{indices }}\\
		\text{are equal to each }j\in\left[N\right]
	}
}\underbrace{\sum_{\substack{0\leq n_{1,1},\dots,n_{r,N}\leq N_{A\bigtriangleup B}+2N_{A\cap B}\\
			\forall\alpha\in\left[r\right]\,\sum_{j=1}^{N}n_{\alpha,j}=n_{\alpha}\\
			\forall j\in\left[N\right]\,\sum_{\alpha=1}^{r}n_{\alpha,j}=m_{j}
		}
}}_{\text{How to distribute the token powers between }\left[N\right]}\\\underbrace{\sum_{\substack{0\leq p_{1,1},\dots,p_{r_{e},N}\leq2C\left(L\right)+1-N_{A\bigtriangleup B}-2N_{A\cap B}\\
			\forall\sigma\in\left[r_{e}\right]\,\sum_{j=1}^{N}p_{\sigma,j}=p_{\sigma}\\
			\forall j\in\left[N\right]\,\sum_{\sigma=1}^{r_{e}}p_{\sigma,j}=z_{j}
		}
}}_{\text{How to distribute the pos powers between }\left[N\right]}\lambda_{N_{A\bigtriangleup B}N_{A\cap B},n_{1},\dots,n_{r},p_{1},\dots,p_{r_{e}}}\left(\prod_{j=1}^{N}\prod_{\sigma=1}^{r_{e}}\left(M_{\sigma,j}^{\text{pos}}\right)^{p_{\sigma,j}}\right)\left(\prod_{j=1}^{N}\prod_{\alpha=1}^{r}\left(M_{\alpha,w^{j}}^{\text{vocab}}\right)^{n_{\alpha,j}}\right)
	\end{align*}
	Where
	\begin{align*}
\lambda_{N_{A\bigtriangleup B}N_{A\cap B},n_{1},\dots,n_{r},p_{1},\dots,p_{r_{e}}}\coloneqq\sum_{\substack{\substack{I_{A}\subseteq\left[C\left(L\right)+1\right]\\
			\substack{I_{B}\subseteq\left[C\left(L\right)\right]}
			\\
			\left|I_{A}\bigtriangleup I_{B}\right|=N_{A\bigtriangleup B}\\
			\left|I_{A}\cap I_{B}\right|=N_{A\cap B}
		}
	}
}\sum_{\substack{\substack{\alpha_{1},\dots,\alpha_{C\left(L\right)+1}\\
			\beta_{1},\dots,\beta_{C\left(L\right)}
		}
		=1\\
		\forall\delta\in\left[r\right]\,\left|\left\{ c\in I_{A}\left|\alpha_{c}=\delta\right.\right\} \right|+\left|\left\{ c\in I_{B}\left|\beta_{c}=\delta\right.\right\} \right|=n_{\delta}
	}
}^{r}\\\sum_{\substack{\substack{\sigma_{1},\dots,\sigma_{C\left(L\right)+1}\\
			\mu_{1},\dots,\mu_{C\left(L\right)}
		}
		=1\\
		\forall\delta\in\left[r_{e}\right]\,\left|\left\{ c\in\left[C\left(L\right)+1\right]\setminus I_{A}\left|\sigma_{c}=\delta\right.\right\} \right|+\left|\left\{ c\in\left[C\left(L\right)\right]\setminus I_{B}\left|\mu_{c}=\delta\right.\right\} \right|=p_{\delta}
	}
}^{r_{e}}\tau_{\substack{I_{A},I_{B}}
	,\alpha_{1},\dots,\mu_{C\left(L\right)}}
	\end{align*}
	Now we can divide the powers between $P,Q$ in the following way:
	\begin{align*}
&=\sum_{N_{A\bigtriangleup B}=0}^{C\left(L\right)+1}\sum_{N_{A\cap B}=0}^{C\left(L\right)}\sum_{\substack{n_{1}+\dots+n_{r}=N_{A\bigtriangleup B}+2N_{A\cap B}\\
		p_{1}+\dots+p_{r_{e}}=2C\left(L\right)+1-N_{A\bigtriangleup B}-2N_{A\cap B}
	}
}\underbrace{\sum_{\substack{m_{P}+m_{Q}=N_{A\bigtriangleup B}+2N_{A\cap B}\\
			z_{P}+z_{Q}=2C\left(L\right)+1-N_{A\bigtriangleup B}-2N_{A\cap B}\\
			\forall j\in\left\{ P,Q\right\} \quad m_{j}+z_{j}\equiv\begin{cases}
				\scriptstyle
				1\,\text{mod 2} & i\in j\\
				\scriptstyle
				0\,\text{mod 2} & i\notin j
			\end{cases}
		}
}}_{\substack{\text{\text{How many }indices }\\
		\text{are in }P\text{ and in }Q
	}
}\underbrace{\sum_{\substack{0\leq n_{1,P},\dots,n_{r,Q},p_{1,P},\dots,p_{r,Q}\leq2C\left(L\right)+1\\
			\forall\alpha\in\left[r\right]\,n_{\alpha,P}+n_{\alpha,Q}=n_{\alpha}\\
			\forall j\in\left\{ P,Q\right\} \,\sum_{\alpha=1}^{r}n_{\alpha,j}=m_{j}\wedge\sum_{\sigma=1}^{r_{e}}p_{\sigma,j}=z_{j}\\
			\forall\sigma\in\left[r_{e}\right]\,p_{\sigma,P}+p_{\sigma,Q}=p_{\sigma}
		}
}}_{\substack{\text{How to distribute the powers between }P\text{ and in }Q}
}\\&\lambda_{N_{A\bigtriangleup B}N_{A\cap B},n_{1},\dots,n_{r},p_{1},\dots,p_{r_{e}}}\chi_{P}\chi_{Q}
	\end{align*}
	Where $\chi_{P},\chi_{Q}$ are functions of $P,Q$ that defined as:
	\begin{align*}
\chi_{T}\coloneqq\sum_{\substack{\left(m_{j}\right)_{j\in T}\in\left[m_{T}\right]\cup\left\{ 0\right\} \\
		\left(z_{j}\right)_{j\in T}\in\left[z_{T}\right]\cup\left\{ 0\right\} \\
		\sum_{j\in T}m_{j}=m_{T}\wedge\sum_{j\in T}z_{j}=z_{T}\\
		\forall j\in T\,m_{j}+z_{j}\equiv\begin{cases}
			\scriptstyle
			1\,\text{mod}\,2 & i=j\\
			\scriptstyle
			0\,\text{mod}\,2 & i\neq j
		\end{cases}
	}
}\sum_{\substack{\left(n_{\alpha,j}\right)_{\alpha,j\in T\times\left[r\right]},\left(p_{\alpha,j}\right)_{\alpha,j\in T\times\left[r_{e}\right]}\in\left[2C\left(L\right)+1\right]\cup\left\{ 0\right\} \\
		\forall\alpha\in\left[r\right]\,\sum_{j\in T}n_{\alpha,j}=n_{\alpha,T}\\
		\forall\alpha\in\left[r_{e}\right]\,\sum_{j\in T}p_{\alpha,j}=p_{\alpha,T}\\
		\forall j\in T\,\sum_{\alpha=1}^{r}n_{\alpha,j}=m_{j,T}\wedge\sum_{\alpha=1}^{r_{e}}p_{\alpha,j}=z_{j,T}
	}
}\prod_{j\in T}\left(\prod_{\alpha=1}^{r}\left(M_{\alpha,w^{j}}^{\text{vocab}}\right)^{n_{m,j}}\right)\left(\prod_{\sigma=1}^{r_{e}}\left(M_{\sigma,j}^{\text{pos}}\right)^{p_{\sigma,j}}\right)
	\end{align*}
	Thus, since each summand is of separation rank $1$ , the separation rank of $y^{i,L, d_x,H, r}_p$ is bounded by the
	number of summands:
	\begin{align*}
\sum_{N_{A\bigtriangleup B}=0}^{C\left(L\right)+1}\sum_{N_{A\cap B}=0}^{C\left(L\right)}\sum_{\substack{n_{1}+\dots+n_{r}=N_{A\bigtriangleup B}+2N_{A\cap B}\\
		p_{1}+\dots+p_{r_{e}}=2C\left(L\right)+1-N_{A\bigtriangleup B}-2N_{A\cap B}
	}
}\underbrace{\sum_{\substack{m_{P}+m_{Q}=N_{A\bigtriangleup B}+2N_{A\cap B}\\
			z_{P}+z_{Q}=2C\left(L\right)+1-N_{A\bigtriangleup B}-2N_{A\cap B}\\
			\forall j\in\left\{ P,Q\right\} \quad m_{j}+z_{j}\equiv\begin{cases}
				{\scriptstyle 1\,\text{mod 2}} & i\in j\\
				{\scriptstyle 0\,\text{mod 2}} & i\notin j
			\end{cases}
		}
}}_{\substack{\text{\text{How many }indices }\\
		\text{are in }P\text{ and in }Q
	}
}\underbrace{\left(\prod_{\alpha=1}^{r}\multiset{2}{n_{\alpha}}\right)\left(\prod_{\sigma=1}^{r_{e}}\multiset{2}{p_{\sigma}}\right)}_{\text{How to distribute the powers between }P\text{ and in }Q}\\\le\underbrace{\multiset{r+r_{e}}{2C\left(L\right)+1}}_{\substack{\text{ways to divide the powers }\\
		\text{between the coordinates}
	}
}\underbrace{\multiset{4}{2C\left(L\right)+1}}_{\substack{\text{ways to divide the indices }\\
		\text{between }P\text{ and in }Q
	}
}\underbrace{\left(\frac{2C\left(L\right)+1}{r}+1\right)^{r}\left(\frac{2C\left(L\right)+1}{r_{e}}+1\right)^{r_{e}}}_{\text{How to distribute the powers between }P\text{ and in }Q}	
	\end{align*}
	where the inequality followed from lemma~\ref{lemma:means_inequality}.
\end{proof}
From here, the upper bound in theorem~\ifdefined\NOBODY 1\else\ref{theorem:embedding_rank_bottleneck_upper_bound}\fi ~of the main text for vocabulary based embedding follows by lemma~\ref{lemma:multiset_upper_bound} with an additional assumption that $r_e=1$.  This assumption is reasonable since successful models such as T5 \cite{raffel2019exploring} use rank $1$ positional embeddings. %
Moreover, in order to verify the validly of this assumption for our setting in practice,  in subsection~\ref{sec:low_rank_positional_exps} we show that the degradation in performance of models with a low rank positional embedding matrix is  much smaller than the degradation caused by the analyzed bottleneck effects.
\subsection{Convolution based embedding}\label{sec:conv_uuper_bound}
In the following theorem, we show how an upper bound on the separation rank is implied by the rank of convolution based embedding. The proof uses similar techniques to the ones used in the previous subsection with some modifications due to the first convolutional layer.
\begin{theorem}\label{theorem:conv_uuper_bound}
	Let $y^{i,L, d_x,H,r}_p$ be the scalar function computing the $p^{\textrm{th}}$ entry of an output vector at position $i\in[N]$ of the $H$-headed depth-$L$ width-$d_x$ Transformer network defined in eq.~\ifdefined\NOBODY 2\else\ref{eq:conv}\fi ~and ~\ifdefined\NOBODY 5\else\ref{eq:our_layer}\fi ~of the main text, where the embedding rank $r$ is defined by eq.~\ifdefined\NOBODY 4\else\ref{eq:conv_r}\fi ~of the main text.
	Let $r_{e}$ denote the rank of the positional embedding matrix and $sep\left(y^{i,L, d_x,H,r}_p\right)$ denote its separation rank \wrt ~any partition $P\cupdot Q=\left[M\right]$ that does not split any patch. 
	Then the following holds:
	\begin{equation}\label{eq:conv_uuper_bound}
		sep(y^{i,L, d_x,H, r}_p)\le\multiset{r+r_e}{\cdot 3^L}\multiset{4}{3^L}\left(3^L+1\right)^{r+r_e}
	\end{equation}
\end{theorem}
\begin{proof}
	By the embedding low-rank assumptions, there exists  $M^{\text{conv }}\in\mathbb{R}^{\frac{M}{N}\times r\times d_{\text{input}}}, M^{\text{low-rank }}\in\mathbb{R}^{d_{x}\times r}, M^{\text{pos}}\in\mathbb{R}^{N\times r_{e}}$ and $P^{\text{low-rank}}\in\mathbb{R}^{d_{x}\times r_{e}}$ such that:
	\begin{align}\label{eq:low_rank_conv}
		\mathbf{y}^{0,i}=\sum_{k=1}^{\frac{M}{N}}M^{\text{low-rank }}M_{k}^{\text{conv }}\mathbf{x}^{\frac{M}{N}\cdot(i-1)+k}+P^{\text{low-rank}}M_{i}^{\text{pos}}
	\end{align}
	We can begin by substituting $\mathbf{y}^{0,i}$ in eq~\ref{eq:gl_explicit_form} (for convenience, we denote $j_{C\left(L\right)+1}\coloneqq i$):
	\begin{align*}
\ensuremath{y_{p}^{i,L,d_{x},H,\Theta}}=\sum_{j_{1},\dots,j_{C\left(L\right)}=1}^{N}\sum_{h\in\left[H\right]^{\left[C\left(L\right)\right]}}\sum_{r_{1},\dots,r_{C\left(L\right)+1}=1}^{d_{a}}B_{r_{1},p}^{\left(0,h\right)}\left(\prod_{c=1}^{C\left(L\right)+1}\left\langle A_{r_{c}}^{\left(c,h\right)},\sum_{k=1}^{\frac{M}{N}}M^{\text{low-rank }}M_{k}^{\text{conv }}\mathbf{x}^{\frac{M}{N}\cdot(j_{c}-1)+k}+P^{\text{low-rank}}M_{j_{c}}^{\text{pos}}\right\rangle \right)\\\left(\prod_{c=1}^{C\left(L\right)}\left\langle B_{r_{c+1}}^{\left(c,h\right)},\sum_{k=1}^{\frac{M}{N}}M^{\text{low-rank }}M_{k}^{\text{conv }}\mathbf{x}^{\frac{M}{N}\cdot(j_{c}-1)+k}+P^{\text{low-rank}}M_{j_{c}}^{\text{pos}}\right\rangle \right)
	\end{align*}
	And separating between the tokens and the positional embeddings: 
	\begin{align*}
=\underbrace{\sum_{\substack{\substack{I_{A}\subseteq\left[C\left(L\right)+1\right]\\
				\substack{I_{B}\subseteq\left[C\left(L\right)\right]}
			}
		}
}}_{\text{the indices of tokens}}\sum_{j_{1},\dots,j_{C\left(L\right)}=1}^{N}\sum_{h\in\left[H\right]^{\left[C\left(L\right)\right]}}\sum_{r_{1},\dots,r_{C\left(L\right)+1}=1}^{d_{a}}\underbrace{\left(\prod_{c\in\left[C\left(L\right)+1\right]\setminus I_{A}}\left\langle A_{r_{c}}^{\left(c,h\right)},P^{\text{low-rank}}M_{j_{c}}^{\text{pos}}\right\rangle \right)\left(\prod_{c\in\left[C\left(L\right)\right]\setminus I_{B}}\left\langle B_{r_{c+1}}^{\left(c,h\right)},P^{\text{low-rank}}M_{j_{c}}^{\text{pos}}\right\rangle \right)}_{\text{The positional embeddings}}\\B_{r_{1},p}^{\left(0,h\right)}\underbrace{\left(\prod_{c\in I_{A}}\left\langle A_{r_{c}}^{\left(c,h\right)},\sum_{k=1}^{\frac{M}{N}}M^{\text{low-rank }}M_{k}^{\text{conv }}\mathbf{x}^{\frac{M}{N}\cdot(j_{c}-1)+k}\right\rangle \right)\left(\prod_{c\in I_{B}}\left\langle B_{r_{c+1}}^{\left(c,h\right)},\sum_{k=1}^{\frac{M}{N}}M^{\text{low-rank }}M_{k}^{\text{conv }}\mathbf{x}^{\frac{M}{N}\cdot(j_{c}-1)+k}\right\rangle \right)}_{\text{The tokens}}
	\end{align*}
	Now we can open the inner products, explicitly writing the indices: 
	\begin{align*}
=\sum_{\substack{\substack{I_{A}\subseteq\left[C\left(L\right)+1\right]\\
			\substack{I_{B}\subseteq\left[C\left(L\right)\right]}
		}
	}
}\sum_{\substack{\alpha_{1},\dots,\alpha_{C\left(L\right)+1}\\
		\beta_{1},\dots,\beta_{C\left(L\right)}
	}
	=1}^{r}\sum_{\substack{\kappa_{1},\ldots,\kappa_{C\left(L\right)+1}\\
		\eta_{1},\ldots,\eta_{C\left(L\right)}
	}
	=1}^{\frac{M}{N}}\sum_{\substack{\sigma_{1},\dots,\sigma_{C\left(L\right)+1}\\
		\mu_{1},\dots,\mu_{C\left(L\right)}
	}
	=1}^{r_{e}}\sum_{j_{1},\dots,j_{C\left(L\right)}=1}^{N}\sum_{h\in\left[H\right]^{\left[C\left(L\right)\right]}}\sum_{r_{1},\dots,r_{C\left(L\right)+1}=1}^{d_{a}}B_{r_{1},p}^{\left(0,h\right)}\\\left(\sum_{\substack{\gamma_{1},\dots,\gamma_{C\left(L\right)+1}\\
		\delta_{1},\dots,\delta_{C\left(L\right)}
	}
	=1}^{d_{x}}\left(\prod_{c\in\left[C\left(L\right)+1\right]\setminus I_{A}}A_{r_{c},\gamma_{c}}^{\left(c,h\right)}P_{\gamma_{c},\sigma_{c}}^{\text{low-rank}}M_{\sigma_{c},j_{c}}^{\text{pos}}\right)\left(\prod_{c\in\left[C\left(L\right)\right]\setminus I_{B}}B_{r_{c+1},\delta_{c}}^{\left(c,h\right)}P_{\delta_{c},\mu_{c}}^{\text{low-rank}}M_{\mu_{c},j_{c}}^{\text{pos}}\right)\right)\\\left(\sum_{\substack{\gamma_{1},\dots,\gamma_{C\left(L\right)+1}\\
		\delta_{1},\dots,\delta_{C\left(L\right)}
	}
	=1}^{d_{x}}\left(\prod_{c\in I_{A}}A_{r_{c},\gamma_{c}}^{\left(c,h\right)}M_{\gamma_{c},\alpha_{c}}^{\text{low-rank}}\left(M_{\kappa_{c}}^{\text{conv }}\mathbf{x}^{\frac{M}{N}\cdot(j_{c}-1)+\kappa_{c}}\right)_{\alpha_{c}}\right)\left(\prod_{c\in I_{B}}B_{r_{c+1},\delta_{c}}^{\left(c,h\right)}M_{\delta_{c},\beta_{c}}^{\text{low-rank}}\left(M_{\eta_{c}}^{\text{conv }}\mathbf{x}^{\frac{M}{N}\cdot(j_{c}-1)+\eta_{c}}\right)_{\beta_{c}}\right)\right)
	\end{align*}
	And separating between coefficients and embeddings: 
	\begin{align*}
=\sum_{\substack{\substack{I_{A}\subseteq\left[C\left(L\right)+1\right]\\
			\substack{I_{B}\subseteq\left[C\left(L\right)\right]}
		}
	}
}\sum_{\substack{\alpha_{1},\dots,\alpha_{C\left(L\right)+1}\\
		\beta_{1},\dots,\beta_{C\left(L\right)}
	}
	=1}^{r}\sum_{\substack{\kappa_{1},\ldots,\kappa_{C\left(L\right)+1}\\
		\eta_{1},\ldots,\eta_{C\left(L\right)}
	}
	=1}^{\frac{M}{N}}\sum_{\substack{\sigma_{1},\dots,\sigma_{C\left(L\right)+1}\\
		\mu_{1},\dots,\mu_{C\left(L\right)}
	}
	=1}^{r_{e}}\sum_{j_{1},\dots,j_{C\left(L\right)}=1}^{N}\tau_{\substack{I_{A},I_{B}}
	,\alpha_{1},\dots,\mu_{C\left(L\right)}}\\\left(\prod_{c\in\left[C\left(L\right)+1\right]\setminus I_{A}}M_{\sigma_{c},j_{c}}^{\text{pos}}\right)\left(\prod_{c\in\left[C\left(L\right)\right]\setminus I_{B}}M_{\mu_{c},j_{c}}^{\text{pos}}\right)\left(\prod_{c\in I_{A}}\left(M_{\kappa_{c}}^{\text{conv }}\mathbf{x}^{\frac{M}{N}\cdot(j_{c}-1)+\kappa_{c}}\right)_{\alpha_{c}}\right)\left(\prod_{c\in I_{B}}\left(M_{\eta_{c}}^{\text{conv }}\mathbf{x}^{\frac{M}{N}\cdot(j_{c}-1)+\eta_{c}}\right)_{\beta_{c}}\right)
	\end{align*}
	Where the coefficients are equal to: 
	\begin{align*}
\tau_{\substack{I_{A},I_{B}}
	,\alpha_{1},\dots,\mu_{C\left(L\right)}}&\coloneqq\sum_{h\in\left[H\right]^{\left[C\left(L\right)\right]}}\sum_{r_{1},\dots,r_{C\left(L\right)+1}=1}^{d_{a}}B_{r_{1},p}^{\left(0,h\right)}\left[\sum_{\substack{\gamma_{1},\dots,\gamma_{C\left(L\right)+1}\\
		\delta_{1},\dots,\delta_{C\left(L\right)}
	}
	=1}^{d_{x}}\left(\prod_{c\in I_{A}}A_{r_{c},\gamma_{c}}^{\left(c,h\right)}M_{\gamma_{c},\alpha_{c}}^{\text{low-rank}}\right)\left(\prod_{c\in I_{B}}B_{r_{c+1},\delta_{c}}^{\left(c,h\right)}M_{\delta_{c},\beta_{c}}^{\text{low-rank}}\right)\right]\\&\left[\sum_{\substack{\gamma_{1},\dots,\gamma_{C\left(L+1\right)}\\
		\delta_{1},\dots,\delta_{C\left(L\right)}
	}
	=1}^{d_{x}}\left(\prod_{c\in\left[C\left(L\right)+1\right]\setminus I_{A}}A_{r_{c},\gamma_{c}}^{\left(c,h\right)}P_{\gamma_{c},\sigma_{c}}^{\text{low-rank}}\right)\left(\prod_{c\in\left[C\left(L\right)\right]\setminus I_{B}}B_{r_{c+1},\delta_{c}}^{\left(c,h\right)}P_{\delta_{c},\mu_{c}}^{\text{low-rank}}\right)\right]\\
	\end{align*}
	Now we can group monomials by the powers $n_{1},\dots,n_{r},p_{1},\dots,p_{r_{e}}$ of each coordinate: 
	\begin{align*}
&=\underbrace{\sum_{N_{A\bigtriangleup B}=0}^{C\left(L\right)+1}\sum_{N_{A\cap B}=0}^{C\left(L\right)}}_{\text{\ensuremath{\substack{\text{How many }j_{c}\text{ indices}\\
				\text{are token indices}
			}
}}}\underbrace{\sum_{\substack{n_{1}+\dots+n_{r}=N_{A\bigtriangleup B}+2N_{A\cap B}\\
			p_{1}+\dots+p_{r_{e}}=2C\left(L\right)+1-N_{A\bigtriangleup B}-2N_{A\cap B}
		}
}}_{\text{The powers}}\underbrace{\sum_{\substack{m_{1,1}+\dots+m_{N,\frac{M}{N}}=N_{A\bigtriangleup B}+2N_{A\cap B}\\
			z_{1,1}+\dots+z_{N,\frac{M}{N}}=2C\left(L\right)+1-N_{A\bigtriangleup B}-2N_{A\cap B}\\
			\forall j\in\left[N\right]\,\sum_{k=1}^{\frac{M}{N}}\left(m_{j,k}+z_{j,k}\right)\equiv\begin{cases}
				{\scriptstyle 1\,\text{mod}\,2} & i=j\\
				{\scriptstyle 0\,\text{mod}\,2} & i\neq j
			\end{cases}
		}
}}_{\substack{\text{\text{How many indices}}\\
		\text{are equal to each }\left(j,k\right)\in\left[N\right]\times\left[\frac{M}{N}\right]
	}
}\\&\underbrace{\sum_{\substack{0\leq n_{1,1,1},\dots,n_{r,N,\frac{M}{N}}\leq N_{A\bigtriangleup B}+2N_{A\cap B}\\
			\forall\alpha\in\left[r\right]\,\sum_{j=1}^{N}\sum_{k=1}^{\frac{M}{N}}n_{\alpha,j,k}=n_{\alpha}\\
			\forall\left(j,k\right)\in\left[N\right]\times\left[\frac{M}{N}\right]\,\sum_{\alpha=1}^{r}n_{\alpha,j,k}=m_{j,k}
		}
}}_{\text{How to distribute the pixel powers between }\left[N\right]\times\left[\frac{M}{N}\right]}\underbrace{\sum_{\substack{0\leq p_{1,1,1},\dots,p_{r_{e},N,\frac{M}{N}}\leq2C\left(L\right)+1-N_{A\bigtriangleup B}-2N_{A\cap B}\\
			\forall\sigma\in\left[r_{e}\right]\,\sum_{j=1}^{N}\sum_{k=1}^{\frac{M}{N}}p_{\sigma,j,k}=p_{\sigma}\\
			\forall\left(j,k\right)\in\left[N\right]\times\left[\frac{M}{N}\right]\,\sum_{\sigma=1}^{r_{e}}p_{\sigma,j,k}=z_{j,k}
		}
}}_{\text{How to distribute the pos powers between }\left[N\right]\times\left[\frac{M}{N}\right]}\\&\Gamma_{z_{1,1},\ldots,z_{N,\frac{M}{N}},p_{1,1,1},\dots,p_{r,N,\frac{M}{N}}}\lambda_{N_{A\bigtriangleup B}N_{A\cap B},n_{1},\dots,n_{r},p_{1},\dots,p_{r_{e}}}\left(\prod_{j=1}^{N}\prod_{k=1}^{\frac{M}{N}}\prod_{\sigma=1}^{r_{e}}\left(M_{\sigma,j}^{\text{pos}}\right)^{p_{\sigma,j,k}}\right)\left(\prod_{j=1}^{N}\prod_{k=1}^{\frac{M}{N}}\prod_{\alpha=1}^{r}\left(\left(M_{k}^{\text{conv }}\mathbf{x}^{\frac{M}{N}\cdot(j-1)+k}\right)_{\alpha}\right)^{n_{\alpha,j,k}}\right)
	\end{align*}
	where 
	\begin{align*}
\lambda_{N_{A\bigtriangleup B}N_{A\cap B},n_{1},\dots,n_{r},p_{1},\dots,p_{r_{e}}}\coloneqq\sum_{\substack{\substack{I_{A}\subseteq\left[C\left(L\right)+1\right]\\
			\substack{I_{B}\subseteq\left[C\left(L\right)\right]}
			\\
			\left|I_{A}\bigtriangleup I_{B}\right|=N_{A\bigtriangleup B}\\
			\left|I_{A}\cap I_{B}\right|=N_{A\cap B}
		}
	}
}\sum_{\substack{\substack{\alpha_{1},\dots,\alpha_{C\left(L\right)+1}\\
			\beta_{1},\dots,\beta_{C\left(L\right)}
		}
		=1\\
		\forall\delta\in\left[r\right]\,\left|\left\{ c\in I_{A}\left|\alpha_{c}=\delta\right.\right\} \right|+\left|\left\{ c\in I_{B}\left|\beta_{c}=\delta\right.\right\} \right|=n_{\delta}
	}
}^{r}\;\sum_{\substack{\substack{\kappa_{1},\ldots,\kappa_{C\left(L\right)+1}\\
			\eta_{1},\ldots,\eta_{C\left(L\right)}
		}
		=1\\
		\forall\delta\in\left[\frac{M}{N}\right]\,\left|\left\{ c\in I_{A}\left|\kappa_{c}=\delta\right.\right\} \right|+\left|\left\{ c\in I_{B}\left|\eta_{c}=\delta\right.\right\} \right|=n_{\delta}
	}
}^{\frac{M}{N}}\\\sum_{\substack{\substack{\sigma_{1},\dots,\sigma_{C\left(L\right)+1}\\
			\mu_{1},\dots,\mu_{C\left(L\right)}
		}
		=1\\
		\forall\delta\in\left[r_{e}\right]\,\left|\left\{ c\in\left[C\left(L\right)+1\right]\setminus I_{A}\left|\sigma_{c}=\delta\right.\right\} \right|+\left|\left\{ c\in\left[C\left(L\right)\right]\setminus I_{B}\left|\mu_{c}=\delta\right.\right\} \right|=p_{\delta}
	}
}^{r_{e}}\tau_{\substack{I_{A},I_{B}}
	,\alpha_{1},\dots,\mu_{C\left(L\right)}}
	\end{align*}
	and
	\begin{align*}
		\Gamma_{z_{1,1},\ldots,z_{N,\frac{M}{N}},p_{1,1,1},\dots,p_{r,N,\frac{M}{N}}}:=\left(\prod_{j=1}^{N}\left[\multiset{\frac{M}{N}}{z_{j,1}+\ldots+z_{j,\frac{M}{N}}}\cdot\prod_{\sigma=1}^{r_{e}}\multiset{\frac{M}{N}}{p_{\sigma,j,1}+\ldots+p_{\sigma,j,\frac{M}{N}}}\right]\right)^{-1}
	\end{align*}		
	Note that the positional powers are actually independent of the indices in $\left[\frac{M}{N}\right]$, so $\Gamma_{z_{1,1},\ldots,z_{N,\frac{M}{N}},p_{1,1,1},\dots,p_{r,N,\frac{M}{N}}}$ is a multiplicative factor that is used in order to cancel out double counting.
	
	For convenience, we will treat $\left(P,Q\right)$ as a partition of the Cartesian product $\left[N\right]\times\left[\frac{M}{N}\right]$ (as there is a one-to-one correspondence between $\left[N\right]\times\left[\frac{M}{N}\right]$ and $\left[M\right])$. Now we can divide the powers between $P,Q$ in the following way: 
	\begin{align*}
=\sum_{N_{A\bigtriangleup B}=0}^{C\left(L\right)+1}\sum_{N_{A\cap B}=0}^{C\left(L\right)}\sum_{\substack{n_{1}+\dots+n_{r}=N_{A\bigtriangleup B}+2N_{A\cap B}\\
		p_{1}+\dots+p_{r_{e}}=2C\left(L\right)+1-N_{A\bigtriangleup B}-2N_{A\cap B}
	}
}\underbrace{\sum_{\substack{m_{P}+m_{Q}=N_{A\bigtriangleup B}+2N_{A\cap B}\\
			z_{P}+z_{Q}=2C\left(L\right)+1-N_{A\bigtriangleup B}-2N_{A\cap B}\\
			\forall j\in\left\{ P,Q\right\} \,m_{j}+z_{j}=\begin{cases}
				\scriptstyle
				1\,\text{mod}\,2 & \exists k\in\left[\frac{M}{n}\right]\,\left(i,k\right)\in j\\
				\scriptstyle
				0\,\text{mod}\,2 & \text{else}
			\end{cases}
		}
}}_{\substack{\text{\text{How many }indices }\\
		\text{are in }P\text{ and in }Q
	}
}\\\underbrace{\sum_{\substack{0\leq n_{1,P},\dots,n_{r,Q},p_{1,P},\dots,p_{r,Q}\leq2C\left(L\right)+1\\
			\forall\alpha\in\left[r\right]\,n_{\alpha,P}+n_{\alpha,Q}=n_{\alpha}\\
			\forall T\in\left\{ P,Q\right\} \,\sum_{\alpha=1}^{r}n_{\alpha,T}=m_{T}\wedge\sum_{\sigma=1}^{r_{e}}p_{\sigma,T}=z_{T}\\
			\forall\sigma\in\left[r_{e}\right]\,p_{\sigma,P}+p_{\sigma,Q}=p_{\sigma}
		}
}}_{\substack{\text{How to distribute the powers between }P\text{ and in }Q}
}\Gamma_{z_{1,1},\ldots,z_{N,\frac{M}{N}},p_{1,1,1},\dots,p_{r,N,\frac{M}{N}}}\lambda_{N_{A\bigtriangleup B}N_{A\cap B},n_{1},\dots,n_{r},p_{1},\dots,p_{r_{e}}}\chi_{P}\chi_{Q}
	\end{align*}
	Where $\chi_{P},\chi_{Q}$ are functions of $P,Q$ that defined as: 
	\begin{align*}
\chi_{T}\coloneqq\sum_{\substack{\left(m_{j,k}\right)_{\left(j,k\right)\in T}\in\left[m_{T}\right]\cup\left\{ 0\right\} \\
		\left(z_{j,k}\right)_{\left(j,k\right)\in T}\in\left[z_{T}\right]\cup\left\{ 0\right\} \\
		\sum_{\left(j,k\right)\in T}m_{j,k}=m_{T}\wedge\sum_{\left(j,k\right)\in T}z_{j,k}=z_{T}\\
		\forall j\in\left\{ j\in\left[N\right]:\exists k\in\left[\frac{M}{n}\right]\,\left(j,k\right)\in T\right\} \,\sum_{k=1}^{\frac{M}{N}}\left(m_{j,k}+z_{j,k}\right)\equiv\begin{cases}
			{\scriptstyle 1\,\text{mod}\,2} & i=j\\
			{\scriptstyle 0\,\text{mod}\,2} & i\neq j
		\end{cases}
	}
}\sum_{\substack{\left(n_{\alpha,j,k}\right)_{\left(j,k\right),\alpha\in T\times\left[r\right]},\left(p_{\alpha,j,k}\right)_{\left(j,k\right),\alpha\in T\times\left[r_{e}\right]}\in\left[2C\left(L\right)+1\right]\cup\left\{ 0\right\} \\
		\forall\alpha\in\left[r\right]\,\sum_{\left(j,k\right)\in T}n_{\alpha,j,k}=n_{\alpha,T}\\
		\forall\alpha\in\left[r_{e}\right]\,\sum_{\left(j,k\right)\in T}p_{\alpha,j,k}=p_{\alpha,T}\\
		\forall\left(j,k\right)\in T\,\sum_{\alpha=1}^{r}n_{\alpha,j,k}=m_{j,k}\wedge\sum_{\alpha=1}^{r_{e}}p_{\alpha,j,k}=z_{j,k}
	}
}\\\prod_{\left(j,k\right)\in T}\left(\prod_{\alpha=1}^{r}\left(\left(M_{k}^{\text{conv }}\mathbf{x}^{\frac{M}{N}\cdot(j-1)+k}\right)_{\alpha}\right)^{n_{\alpha,j,k}}\right)\left(\prod_{\sigma=1}^{r_{e}}\left(M_{\sigma,j}^{\text{pos}}\right)^{p_{\sigma,j,k}}\right)
	\end{align*}
	Thus, since each summand is of separation rank $1$ , the separation rank of $y_{p}^{i,L,d_{x},H,r}$ is bounded by the number of summands:
	\begin{align*}
\sum_{N_{A\bigtriangleup B}=0}^{C\left(L\right)+1}\sum_{N_{A\cap B}=0}^{C\left(L\right)}\sum_{\substack{n_{1}+\dots+n_{r}=N_{A\bigtriangleup B}+2N_{A\cap B}\\
		p_{1}+\dots+p_{r_{e}}=2C\left(L\right)+1-N_{A\bigtriangleup B}-2N_{A\cap B}
	}
}\underbrace{\sum_{\substack{m_{P}+m_{Q}=N_{A\bigtriangleup B}+2N_{A\cap B}\\
			z_{P}+z_{Q}=2C\left(L\right)+1-N_{A\bigtriangleup B}-2N_{A\cap B}\\
			\forall j\in\left\{ P,Q\right\} \,m_{j}+z_{j}=\begin{cases}
				\scriptstyle
				{ 1\,\text{mod}\,2} & \exists k\in\left[\frac{M}{n}\right]\,\left(i,k\right)\in j\\
				\scriptstyle
				{ 0\,\text{mod}\,2} & \text{else}
			\end{cases}
		}
}}_{\substack{\text{\text{How many }indices }\\
		\text{are in }P\text{ and in }Q
	}
}\underbrace{\left(\prod_{\alpha=1}^{r}\multiset{2}{n_{\alpha}}\right)\left(\prod_{\sigma=1}^{r_{e}}\multiset{2}{p_{\sigma}}\right)}_{\text{How to distribute the powers between }P\text{ and in }Q}\\\le\underbrace{\multiset{r+r_{e}}{2C\left(L\right)+1}}_{\substack{\text{ways to divide the powers }\\
		\text{between the coordinates}
	}
}\underbrace{\multiset{4}{2C\left(L\right)+1}}_{\substack{\text{ways to divide the indices }\\
		\text{between }P\text{ and in }Q
	}
}\underbrace{\left(\frac{2C\left(L\right)+1}{r}+1\right)^{r}\left(\frac{2C\left(L\right)+1}{r_{e}}+1\right)^{r_{e}}}_{\text{How to distribute the powers between }P\text{ and in }Q}
	\end{align*} 
\end{proof}	
Similarly to the vocabulary embedding case from here, the upper bound in theorem~\ifdefined\NOBODY 1\else\ref{theorem:embedding_rank_bottleneck_upper_bound}\fi ~of the main text for convolution based embedding follows by lemma~\ref{lemma:multiset_upper_bound} with an additional assumption that $r_e=1$.  
\section{Lower bounds on the separation rank}
\subsection{Preliminaries}
\subsubsection{Tensors and their matricization}
We begin by laying out basic concepts in tensor theory required for
the upcoming analysis. The core concept of a \emph{tensor} may be thought of as a
multi-dimensional array. The \emph{order} of a
tensor is defined to be the number of indexing entries in the array,
referred to as \emph{modes}. The \emph{dimension} of a tensor in a particular
mode is defined as the number of values taken by the index in that
mode. If $\A$ is a tensor of order $N$ and dimension $M_i$ in each mode
$i\in[N]$, its entries are denoted $\A_{d_1...d_N}$, where the index in each
mode takes values $d_i\in [M_i]$.

We will make use of the concept of the \emph{matricization of $\A$ \wrt~the balanced partition $(P,Q)$}, denoted $\mat{\A}_{P,Q}\in\R^{M^{\nicefrac{N}{2}}\times M^{\nicefrac{N}{2}}}$, which is essentially the arrangement of the tensor elements as a matrix whose rows correspond to $P$ and columns to $Q$.
Suppose $\A\in\R^{M{\times\cdots\times}M}$
is a tensor of order $N$, and let $(P,Q)$ be a balanced partition of $[N]$, \ie~$P$
and~$Q$ are disjoint size $\nicefrac{N}{2}$ subsets of $[N]$ whose union gives~$[N]$.
The \emph{matricization of $\A$ \wrt~the partition $(P,Q)$}, denoted
$\mat{\A}_{P,Q}$, is the $M^{{\nicefrac{N}{2}}}$-by-$M^{{\nicefrac{N}{2}}}$ matrix holding the entries of $\A$ such that $\A_{d_1{\ldots}d_N}$ is placed in row index $1+\sum_{t=1}^{{\nicefrac{N}{2}}}(d_{p_t}-1)M^{{\nicefrac{N}{2}}-t}$ and column index $1+\sum_{t=1}^{{\nicefrac{N}{2}}}(d_{q_t}-1)M^{{\nicefrac{N}{2}}-t}$.

\subsubsection{Grid tensors provide lower bounds for the separation rank} 

We now present the concept of grid tensors, which are a form of function discretization~\citep{hackbusch2012tensor}. Essentially, the function is
evaluated for a set of points on an exponentially large grid in the
input space and the outcomes are stored in a tensor. Formally, fixing a set of \emph{template} vectors
$\x^{(1)},\ldots,\x^{(Z)}$, either in $\R^{d_{\text{input}}}$ for the convolutional embedding method or in $\left[V\right]$ for the vocabulary embedding method, the points on the grid are the set
$\{(\x^{(d_1)},\ldots,\x^{(d_N)})\}_{d_1,\ldots,d_N=1}^Z$. Given a function
$y(\x^1,\ldots,\x^N)$, the set of its values on the grid arranged in the form of a tensor are
called the grid tensor induced by $y$, denoted
$\A(y)_{d_1,\ldots,d_N} \equiv y(\x^1=\x^{(d_1)},\ldots,\x^N=\x^{(d_N)})$.

Let $T$ denote the number of raw inputs to the network. In the notation of section~\ifdefined\NOBODY 2.1\else\ref{sec:TransformerArchitecture:Embedding}\fi ~of the main text, $T$ is either equal to $N$ in the case of the vocabulary input embedding or $M$ in the case of the convolution input embedding.
The following claim from \cite{levine2020limits} establishes a fundamental relation between a function's separation rank (see section~\ifdefined\NOBODY 3\else\ref{sec:separation_rank}\fi ~of the main text) and the rank of the matrix obtained by the
corresponding grid tensor matricization. This relation, which holds for all functions, is
formulated below for functions realized by the analyzed Transformer network:

\begin{claim}\label{claim:grid_sep_deep}
	Let $y^{i,L, d_x,H,r}_p$ be the scalar function computing the $p^{\textrm{th}}$ entry of an output vector at position $i\in[N]$ of the $H$-headed depth-$L$ width-$d_x$ Transformer network defined in eq.~\ifdefined\NOBODY 5\else\ref{eq:our_layer}\fi ~and either eq~\ifdefined\NOBODY 1\else\ref{eq:vocab}\fi ~or eq~\ifdefined\NOBODY 2\else\ref{eq:conv}\fi ~of the main text.
	Let $sep\left(y^{i,L, d_x,H,r}_p\right)$ denote its separation rank \wrt ~any partition $P\cupdot Q=\left[T\right]$. Then, for any integer $Z$ and any set of template vectors $\x^{(1)},\ldots,\x^{(Z)} \in \R^{d_x}$ it
	holds that:
	\begin{equation}
		\sep{P,Q}{y^{i,L, d_x,H,r}_p}\geq \rank{\mat{\A(y^{i,L, d_x,H,r}_p)}_{P,Q}},
	\end{equation}
	where $\A(y^{i,L, d_x,H,r}_p)$ is the grid tensor of $y^{i,L, d_x,H,r}_p$ with
	respect to the above template vectors.
\end{claim}
In the next subsection we will show a corollary from \cite{levine2020limits} that uses this claim to prove the lower bound in theorem~\ifdefined\NOBODY 2\else\ref{theorem:embedding_rank_bottleneck_lower_bound}\fi ~of the main text.

\subsection{Proof of the lower bounds}\label{sec:lower_bound_proff}
In this subsection we prove the lower bound in theorem~\ifdefined\NOBODY 2\else\ref{theorem:embedding_rank_bottleneck_lower_bound}\fi ~of the main text.
We will use a direct corollary of the proof in  \citet{levine2020limits} regarding composition of the self-attention separation rank.
Essentially, though the required form of $\y^{0,j}$ in corollary below looks complex,~\citet{levine2020limits} prove that for this form of inputs to the self-attention block, the rank of the grid tensor is with probability $1$ lower bounded by the multiset term in eq~\ref{equation:sufficient_assigmnet_actual_lower_bound} below. The corollary below simply states that if the input embedding is able to produce vectors that do not change the analysis in~\cite{levine2020limits}, their bound on the grid tensor rank can be used, and together with claim~\ref{claim:grid_sep_deep} this implies a lower bound on the separation rank.    

Denote by $y^{i,L, d_x,H,r}_p$ the scalar function computing the $p^{\textrm{th}}$ entry of an output vector at position $i\in[N]$ of the $H$-headed depth-$L$ width-$d_x$ Transformer network defined in eq.~\ifdefined\NOBODY 5\else\ref{eq:our_layer}\fi ~and either eq~\ifdefined\NOBODY 1\else\ref{eq:vocab}\fi ~or eq~\ifdefined\NOBODY 2\else\ref{eq:conv}\fi ~of the main text, then:
\begin{corollary}\label{corollary:sufficient_assigmnet}
	Assume that for any matrix $A\in\R^{\multiset{\nicefrac{\left(r-H\right)}{2}}{3^{L-2}}\times \nicefrac{\left(r-H\right)}{2}}$ with rows that are $l^2$ normalized, there exists a choice of template vectors $\x^{(1)},\ldots,\x^{(Z)}$, as well as an assignment to the embedding layer weights, such that for any sequence $\left(i_{j}\right)_{j=1}^{N}\in\left[2\cdot\multiset{\nicefrac{\left(r-H\right)}{2}}{3^{L-2}}+1\right]$ there exists a sequence of $T$ $\x$'s for which the output of the embedding layer is:
	\begin{align*}
		\forall j\in\left[N\right]\quad\y_{\alpha}^{(0,j)}=\begin{cases}
			A_{i_{j},\phi(\alpha)} & i_{j}\leq\nicefrac{V}{2}\wedge(\alpha-1)\bmod d_{a}<\frac{d_{a}-1}{2}\wedge\phi(\alpha)\le\nicefrac{\left(r-H\right)}{2}\\
			A_{i_{j}-\nicefrac{V}{2},\phi\left(\alpha-\frac{d_{a}-1}{2}\right)} & \nicefrac{V}{2}<i_{j}\leq V\wedge\frac{d_{a}-1}{2}\leq(\alpha-1)\bmod d_{a}<d_{a}-1\wedge\phi(\alpha-\frac{d_{a}-1}{2})\le\nicefrac{\left(r-H\right)}{2}\\
			1 & (\alpha-1)\bmod d_{a}=d_{a}-1\\
			0 & \text{Otherwise}
		\end{cases}
	\end{align*}
	where $\phi(j) \equiv \left\lfloor \nicefrac{j - 1}{d_a} \right\rfloor \cdot (d_a - 1) + (j - 1 \bmod d_a) + 1$ and $V:=2\left(\binom{\nicefrac{\left(r-H\right)}{2}}{3^{L-2}}\right)$. 
	
	Further in the convolutional embedding case assume the partition $P\cupdot Q=\left[T\right]$ does not split any patch. Then for all values of the network weights but a set of Lebesgue measure
	zero, the following holds:
	\begin{equation}
		\label{equation:sufficient_assigmnet_actual_lower_bound}
		sep(y^{i,L, d_x,H, r}_p)\ge\multiset{\nicefrac{\left(r-H\right)}{2}}{3^{L-2}}
	\end{equation}
\end{corollary}

Now we will prove that both for the convolutional embedding method and the vocabulary embedding method the assumption of corollary~\ref{corollary:sufficient_assigmnet} holds. We will thus prove the lower bound in theorem~\ifdefined\NOBODY 2\else\ref{theorem:embedding_rank_bottleneck_lower_bound}\fi, since the following lemma~\ref{lemma:multiset_lower_bound} shows that $\log\multiset{\nicefrac{\left(r-H\right)}{2}}{3^{L-2}}=\tilde{\Omega}\left(L\cdot\left(\min\{r,d_x\}-H\right)\right)$.

\begin{lemma}\label{lemma:multiset_lower_bound}
	${\left(\binom{n}{k}\right)\ge\left(\frac{2e\left(n+k\right)}{n}\right)}^{n}$
\end{lemma}
\begin{proof}
	: by using the inequality $\binom{n}{k}\ge\left(\frac{n}{k}\right)^{k}$ we have $\left(\binom{n}{k}\right)=\binom{n+k-1}{n-1}\ge\left(\frac{\left(n+k-1\right)}{n-1}\right)^{n-1}$
\end{proof}

\subsubsection{Convolution based embedding}
We start with the convolutional embedding method. The lemma below shows that the assumption of corollary~\ref{corollary:sufficient_assigmnet} holds, by dividing the desired vector coordinates into chunks of size $d_{\text{input}}$, and using a convolutional kernel to unify these chunks. 
\begin{lemma}\label{lemma:conv_assigmnet}
	Let $A\in\R^{\multiset{\nicefrac{\left(r-H\right)}{2}}{3^{L-2}}\times \nicefrac{\left(r-H\right)}{2}}$ be a matrix with rows that are $l^2$ normalized, then there exists a choice of template vectors $\x^{(1)},\ldots,\x^{(Z)}$, as well as an assignment to the convolutional embedding layer weights, such that for any sequence $\left(i_{j}\right)_{j=1}^{N}\in\left[2\cdot\multiset{\nicefrac{\left(r-H\right)}{2}}{3^{L-2}}+1\right]$ there exists a sequence of $M$ $\x$'s for which the output of the embedding layer is:
	\begin{align}
		\label{eq:conv_assigmnet_embedding_requirements}
		\forall j\in\left[N\right]\quad\y_{\alpha}^{(0,j)}=\begin{cases}
			A_{i_{j},\phi(\alpha)} & i_{j}\leq\nicefrac{V}{2}\wedge(\alpha-1)\bmod d_{a}<\frac{d_{a}-1}{2}\wedge\phi(\alpha)\le\nicefrac{\left(r-H\right)}{2}\\
			A_{i_{j}-\nicefrac{V}{2},\phi\left(\alpha-\frac{d_{a}-1}{2}\right)} & \nicefrac{V}{2}<i_{j}\leq V\wedge\frac{d_{a}-1}{2}\leq(\alpha-1)\bmod d_{a}<d_{a}-1\wedge\phi(\alpha-\frac{d_{a}-1}{2})\le\nicefrac{\left(r-H\right)}{2}\\
			1 & (\alpha-1)\bmod d_{a}=d_{a}-1\\
			0 & \text{Otherwise}
		\end{cases}
	\end{align}
	where $\phi(j) \equiv \left\lfloor \nicefrac{j - 1}{d_a} \right\rfloor \cdot (d_a - 1) + (j - 1 \bmod d_a) + 1$ and $V:=2\left(\binom{\nicefrac{\left(r-H\right)}{2}}{3^{L-2}}\right)$.
\end{lemma}
\begin{proof}
	Denote the convolutional kernel's width by $k\coloneqq\frac{M}{N}$.
	We will define a convolutional kernel, $W^{\text{conv}}\in\mathbb{R}^{k\times d_{x}\times d_{\text{input}}}$,
	a positional embedding matrix $P\in\mathbb{R}^{N\times d_{x}}$, and
	a set $k$ template vectors, $\x^{(j,1)},\ldots,\x^{(j,k)}$ for each
	$j\in\left[N\right]$, such that:
	\[
	\y^{(0,j)}=\left(\sum_{l=1}^{k}W_{l}^{\text{conv}}\x^{\left(j,l\right)}+P_{i}\right)_{\alpha}
	\]
	
	We will assign weights for the convolutional kernel that will "read" only the non zeros coordinates of eq~\ref{eq:conv_assigmnet_embedding_requirements}:
	
	\[
	W_{l,\alpha,\lambda}^{\text{conv}}=\begin{cases}
		1 & k\cdot\left(\lambda-1\right)+l=\alpha\wedge(\alpha-1)\bmod d_{a}<\frac{d_{a}-1}{2}\wedge\phi\left(\alpha\right)\le\nicefrac{\left(r-H\right)}{2}\\
		1 & k\cdot\left(\lambda-1\right)+l=\alpha\wedge\frac{d_{a}-1}{2}\leq(\alpha-1)\bmod d_{a}<d_{a}-1\wedge\phi\left(\alpha-\frac{d_{a}-1}{2}\right)\le\nicefrac{\left(r-H\right)}{2}\\
		1 & k\cdot\left(\lambda-1\right)+l=\alpha\wedge\left(\alpha-1\right)\bmod d_{a}=d_{a}-1\\
		0 & \text{Otherwise}
	\end{cases}
	\]
	where $\psi\left(l,\lambda\right)\equiv\phi\left(k\cdot\left(\lambda-1\right)+l\right)$. Clearly, the rank of the chosen convolutional kernel is (at most) $r$ and thus satisfy the assumption regarding the embedding rank in theorem~\ifdefined\NOBODY 2\else\ref{theorem:embedding_rank_bottleneck_lower_bound}\fi ~of the main text.
	We will set the positional embedding matrix to be: $P\equiv0$,
	and the template vectors will be defined as follows:
	\[
	\forall j\in\left[N\right],l\in\left[k\right]\quad\x_{\lambda}^{(j,l)}=\begin{cases}
		A_{i_{j},\psi\left(l,\lambda\right)} & i_{j}\leq\nicefrac{V}{2}\wedge\left(k\cdot\left(\lambda-1\right)+l-1\right)\bmod d_{a}<\frac{d_{a}-1}{2}\wedge\psi\left(l,\lambda\right)\le\nicefrac{\left(r-H\right)}{2}\\
		A_{i_{j}-\nicefrac{V}{2},\psi\left(l-\left\lfloor \frac{d_{a}-1}{2}\right\rfloor ,\lambda\right)} & \substack{\nicefrac{V}{2}<i_{j}\leq V\wedge\frac{d_{a}-1}{2}\leq\left(k\cdot\left(\lambda-1\right)+l-1\right)\bmod d_{a}<d_{a}-1}
		\wedge\psi\left(l-\left\lfloor \frac{d_{a}-1}{2}\right\rfloor ,\lambda\right)\le\nicefrac{\left(r-H\right)}{2}\\
		1 & \left(k\cdot\left(\lambda-1\right)+l-1\right)\bmod d_{a}=d_{a}-1\\
		0 & \text{Otherwise}
	\end{cases}
	\]
	
	Now, for each $j\in\left[N\right],\alpha\in\left[d_{x}\right]$ we have:
	\begin{align*}
		\y_{\alpha}^{(0,j)}&=\left(\sum_{l=1}^{k}W_{l}^{\text{conv}}\x^{\left(j,l\right)}\right)_{\alpha}=\sum_{l=1}^{k}\sum_{\lambda=1}^{d_{\text{input}}}W_{l,\alpha,\lambda}^{\text{conv}}\x_{\lambda}^{\left(j,l\right)}\\&\stackrel{^{\left(1\right)}}{=}\begin{cases}
			\x_{\left\lfloor \frac{\alpha-1}{k}\right\rfloor +1}^{\left(j,\left(\alpha-1\mod k\right)+1\right)} & \substack{\left((\alpha-1)\bmod d_{a}<\frac{d_{a}-1}{2}\wedge\phi\left(\alpha\right)\le\nicefrac{\left(r-H\right)}{2}\right)\lor\\
				\left(\frac{d_{a}-1}{2}\leq(\alpha-1)\bmod d_{a}<d_{a}-1\wedge\phi\left(\alpha-\frac{d_{a}-1}{2}\right)\le\nicefrac{\left(r-H\right)}{2}\right)\lor\\
				\left((\alpha-1)\bmod d_{a}=d_{a}-1\right)
			}
			\\
			0 & \text{Otherwise}
		\end{cases}\\&\stackrel{^{\left(2\right)}}{=}\begin{cases}
			A_{i_{j},\phi\left(\alpha\right)} & \alpha\leq r\wedge i_{j}\leq\nicefrac{V}{2}\wedge\left(\alpha-1\right)\bmod d_{a}<\frac{d_{a}-1}{2}\wedge\phi\left(\alpha\right)\le\nicefrac{\left(r-H\right)}{2}\\
			A_{i_{j}-\nicefrac{M}{2},\phi\left(\alpha\right)} & \alpha\leq r\wedge\nicefrac{V}{2}<i_{j}\leq V\wedge\frac{d_{a}-1}{2}\leq\left(\alpha-1\right)\bmod d_{a}<d_{a}-1\wedge\phi\left(\alpha-\frac{d_{a}-1}{2}\right)\le\nicefrac{\left(r-H\right)}{2}\\
			1 & \alpha\leq r\wedge\left(\alpha-1\right)\bmod d_{a}=d_{a}-1\\
			0 & \text{Otherwise}
		\end{cases}
	\end{align*}
	Where $^{\left(1\right)}$ due to the fact that there there's a single
	combination of $l\in\left[k\right],\lambda\in\left[d_{\text{input}}\right]$
	that satisfies $\alpha=k\cdot\left(\lambda-1\right)+l$ (for all other
	value of $l$ and $\lambda$, $W_{l,\alpha,\lambda}^{\text{conv}}=0$),
	and $^{\left(2\right)}$ is since: 
	\[
	\alpha=k\cdot\left\lfloor \frac{\alpha-1}{k}\right\rfloor +\left(\alpha-1\mod k\right)+1
	\]
\end{proof}

\subsubsection{Vocabulary based embedding}
Now we move to the vocabulary embedding method, in this case we will use $A$ to create an assignment for $M_{\textrm{V}}\in\R^{d_x\times V}$, since the input is no longer continuous the number of unique inputs is limited to only $V^N$ . To overcome this issue we will add an additional assumption that either $V\ge2\cdot\multiset{\nicefrac{\left(r-H\right)}{2}}{3^{L-2}}+1$ or $N$ is very large.
Importantly, the upper bound for the small vocabulary size holds with small $V$ and $N$, so the bottleneck phenomenon is theoretically established for the vocabulary embedding method also in cases that neither of the additional assumptions hold. 

We start with the $V\ge2\cdot\multiset{\nicefrac{\left(r-H\right)}{2}}{3^{L-2}}+1$ assumption that overcomes the unique inputs issue by enlarging the number of unique tokens (while keeping the rank $r$ constraint of $M_{\textrm{V}}$).
\begin{lemma}\label{lemma:vocab_assigmnet}
	Assume $V\ge2\cdot\multiset{\nicefrac{\left(r-H\right)}{2}}{3^{L-2}}+1$ and let $A\in\R^{\multiset{\nicefrac{\left(r-H\right)}{2}}{3^{L-2}}\times \nicefrac{\left(r-H\right)}{2}}$ be a matrix with rows that are $l^2$ normalized, then there exists a choice of template vectors $\hat{\mathbf{w}}^{(1)},\ldots,\hat{\mathbf{w}}^{(Z)}$, as well as an assignment to the vocabulary embedding layer weights, such that for any sequence $\left(i_{j}\right)_{j=1}^{N}\in\left[2\cdot\multiset{\nicefrac{\left(r-H\right)}{2}}{3^{L-2}}+1\right]$ there exists a sequence of $T$ $\hat{\mathbf{w}}$'s for which the output of the embedding layer is:
	\begin{align*}
		\forall j\in\left[N\right]\quad\y_{\alpha}^{(0,j)}=\begin{cases}
			A_{i_{j},\phi(\alpha)} & i_{j}\leq\nicefrac{E}{2}\wedge(\alpha-1)\bmod d_{a}<\frac{d_{a}-1}{2}\wedge\phi(\alpha)\le\nicefrac{\left(r-H\right)}{2}\\
			A_{i_{j}-\nicefrac{E}{2},\phi\left(\alpha-\frac{d_{a}-1}{2}\right)} & \nicefrac{E}{2}<i_{j}\leq E\wedge\frac{d_{a}-1}{2}\leq(\alpha-1)\bmod d_{a}<d_{a}-1\wedge\phi(\alpha-\frac{d_{a}-1}{2})\le\nicefrac{\left(r-H\right)}{2}\\
			1 & (\alpha-1)\bmod d_{a}=d_{a}-1\\
			0 & \text{Otherwise}
		\end{cases}
	\end{align*}
	where $\phi(j) \equiv \left\lfloor \nicefrac{j - 1}{d_a} \right\rfloor \cdot (d_a - 1) + (j - 1 \bmod d_a) + 1$ and $E:=2\left(\binom{\nicefrac{\left(r-H\right)}{2}}{3^{L-2}}\right)$. 
\end{lemma}
\begin{proof}
	Our templates vectors will be: $\forall i\in\left[E+1\right]\quad w^i\coloneqq i$.
	We will ignore the positional embedding by choosing $\p^i\coloneqq0$ (by the terms of corollary~\ref{corollary:sufficient_assigmnet} it suffices to find any assignment of the learned weights).
	Now we can use $A$ to create an assignment for $M_{\textrm{V}}\in\R^{d_x\times V}$:
	\begin{align*}
		\left(M_{\textrm{V}}\right)_{\alpha,i} &\coloneqq \begin{cases}
			A_{i_{j},\phi(\alpha)} & i_{j}\leq\multiset{\nicefrac{\left(r-H\right)}{2}}{3^{L-2}}\wedge(\alpha-1)\bmod d_{a}<\frac{d_{a}-1}{2}\\
			A_{i_{j}-\multiset{\nicefrac{\left(r-H\right)}{2}}{3^{L-2}},\phi\left(\alpha-\frac{d_{a}-1}{2}\right)} & \multiset{\nicefrac{\left(r-H\right)}{2}}{3^{L-2}}<i_{j}\leq2\multiset{\nicefrac{\left(r-H\right)}{2}}{3^{L-2}}\wedge\frac{d_{a}-1}{2}\leq(\alpha-1)\bmod d_{a}<d_{a}-1\\
			1 & (\alpha-1)\bmod d_{a}=d_{a}-1\\
			0 & \text{Otherwise}
		\end{cases}
	\end{align*}
	Clearly, the rank of $M_{\textrm{V}}\in\R^{d_x\times V}$ is (at most) $r$, since it has at most $r$ non zero rows, and thus satisfy the assumption regarding the embedding rank in theorem~\ifdefined\NOBODY 2\else\ref{theorem:embedding_rank_bottleneck_lower_bound}\fi ~of the main text.
	Now by eq~\ifdefined\NOBODY 1\else\ref{eq:vocab}\fi ~of the main text, for any given sequence $\left(i_{j}\right)_{j=1}^{N}\in\left[2\cdot\multiset{\nicefrac{\left(r-H\right)}{2}}{3^{L-2}}+1\right]$ we get:
	\begin{align*}
		\y^{\left(0,j\right)}_\alpha=\left(M_{\textrm{V}} ~\hat{\mathbf{w}}^{i_j}\right)_\alpha+\p^{i_j}_\alpha=\left(M_{\textrm{V}}\right)_{\alpha,i_j}
	\end{align*}
\end{proof}

Now, we prove a lower bound with $V=r$ for the infinite $N$ limit. Note that while our proof technique requires unpractical $N$ values, its usage of $N$ is clearly wasteful, and we conjecture (and empirically demonstrate in section~\ifdefined\NOBODY 5\else\ref{sec:experiments}\fi) that the upper bound in theorem~\ifdefined\NOBODY 1\else\ref{theorem:embedding_rank_bottleneck_upper_bound}\fi ~of the main text is tight for $N=\Omega\left(\nicefrac{r\cdot L}{\log_{3} r}\right)$\footnote{Since for $N$ that is larger than this bound, the limitation of $V^N$ unique vectors does not constitutes a bottleneck anymore.}.

In this case, the input embedding is unable to produce vectors that do not change the analysis in \cite{levine2020limits}, and therefore the assumption of corollary~\ref{corollary:sufficient_assigmnet} does not holds. Instead we will use the first self-attention layer of the network to take advantage of the larger $N$, and apply corollary~\ref{corollary:sufficient_assigmnet_layer_1} below to prove a lower bound on the separation rank. This corollary which is direct results of the proof in \cite{levine2020limits} and lemma~\ref{lemma:hadamard_power_rank}, simply states that if the output of the first self-attention layer is able to produce vectors that do not change the analysis in \cite{levine2020limits}, their bound on the grid tensor rank can be used, and together with claim~\ref{claim:grid_sep_deep} this implies a lower bound on the separation rank.

\begin{corollary}\label{corollary:sufficient_assigmnet_layer_1}
	Let $d>0$, assume that for any balanced partition of $[T]$, denoted $(P,Q)$, for any matrix $A\in\N^{\multiset{d}{3^{L-2}}\times d}$ with rows that have equal $ l^2$ norm, there exists a choice of template vectors $\x^{(1)},\ldots,\x^{(Z)}$, an assignment to the embedding layer and the first self-attention layer key and query weights, as well as a mapping $\pi_J:\left[\multiset{d}{3^{L-2}}\right]\rightarrow\left(i_j\right)_{j\in J}$, such that for any $j_1,j_2\in\left[\multiset{d}{3^{L-2}}\right]$ the output of the first self-attention layer on the sequence defined by $\pi_P\left(j_1\right),\pi_Q\left(j_2\right)$ is: $$\y^{(1,j)}=\left(\sum_{h=1}^H W^{O,1,h} W^{V,1,h} \right)\uu$$ for	
	\begin{align*}
		\forall \alpha\in\left[d_x\right]\quad\uu_\alpha=\begin{cases}
			A_{j_1,\phi(\alpha)} & (\alpha-1)\bmod d_{a}<\frac{d_{a}-1}{2}\wedge\phi(\alpha)\le d\\
			A_{j_2,\phi\left(\alpha-\frac{d_{a}-1}{2}\right)} & \frac{d_{a}-1}{2}\leq(\alpha-1)\bmod d_{a}<d_{a}-1\wedge\phi(\alpha-\frac{d_{a}-1}{2})\le d\\
			N & (\alpha-1)\bmod d_{a}=d_{a}-1\\
			0 & \text{Otherwise}
		\end{cases}
	\end{align*}
	where $\phi(j) \equiv \left\lfloor \nicefrac{j - 1}{d_a} \right\rfloor \cdot (d_a - 1) + (j - 1 \bmod d_a) + 1$.
	
	Then for all values of the network weights but a set of Lebesgue measure
	zero, the following holds:
	\begin{equation}
		\label{equation:sufficient_assigmnet_layer_1_actual_lower_bound}
		sep(y^{i,L, d_x,H, r}_p)\ge\multiset{d}{3^{L-2}}
	\end{equation}
\end{corollary}

The lemma below shows that the assumption of corollary~\ref{corollary:sufficient_assigmnet_layer_1} holds for $d\coloneqq\nicefrac{\left(r-1-H\right)}{2}$\footnote{For simplicity we assume that $r-H$ is odd \ie ~$d\in\N$, otherwise we can use $\left\lfloor d\right\rfloor$.}, by choosing an assignment to the first self-attention layer that utilize the large $N$ for summing it's inputs embedding, and use $\pi$ to construct sequences that repeat the one-hot embedding vectors amount of times that depends on $A$.
\begin{lemma}\label{lemma:vocab_large_N_assigmnet}
	Assume $V\ge r$ and let $A\in\N^{\multiset{\nicefrac{\left(r-1-H\right)}{2}}{3^{L-2}}\times \nicefrac{\left(r-1-H\right)}{2}}$ be a matrix with rows that have equal $ l^2$ norm, then there exists a choice of template vectors $\hat{\mathbf{w}}^{(1)},\ldots,\hat{\mathbf{w}}^{(Z)}$, large enough $N$, an assignment to the vocabulary embedding layer and the first self-attention layer key and query weights, as well as a mapping $\pi_J:\left[\multiset{\nicefrac{\left(r-1-H\right)}{2}}{3^{L-2}}\right]\rightarrow\left(i_j\right)_{j\in J}$, such that for any $j_1,j_2\in\left[\multiset{\nicefrac{\left(r-1-H\right)}{2}}{3^{L-2}}\right]$ the output of the first self-attention layer on the sequence defined by $\pi_P\left(j_1\right),\pi_Q\left(j_2\right)$ is:
	$$\y^{(1,j)}=\left(\sum_{h=1}^H W^{O,1,h} W^{V,1,h} \right)\uu$$ for	
	\begin{align*}
		\forall \alpha\in\left[d_x\right]\quad\uu_\alpha=\begin{cases}
			A_{j_1,\phi(\alpha)} & (\alpha-1)\bmod d_{a}<\frac{d_{a}-1}{2}\wedge\phi(\alpha)\le \nicefrac{\left(r-1-H\right)}{2}\\
			A_{j_2,\phi\left(\alpha-\frac{d_{a}-1}{2}\right)} & \frac{d_{a}-1}{2}\leq(\alpha-1)\bmod d_{a}<d_{a}-1\wedge\phi(\alpha-\frac{d_{a}-1}{2})\le \nicefrac{\left(r-1-H\right)}{2}\\
			N & (\alpha-1)\bmod d_{a}=d_{a}-1\\
			0 & \text{Otherwise}
		\end{cases}
	\end{align*}
	where $\phi(j) \equiv \left\lfloor \nicefrac{j - 1}{d_a} \right\rfloor \cdot (d_a - 1) + (j - 1 \bmod d_a) + 1$. 
\end{lemma}
\begin{proof}
	Our templates vectors will be: $\forall i\in\left[r\right]\quad w^i\coloneqq i$.
	We will ignore the positional embedding by choosing $\p^i\coloneqq0$ (by the terms of corollary~\ref{corollary:sufficient_assigmnet_layer_1} it suffices to find any assignment of the learned weights).
	
	To implement summation of the inputs embedding in the first self-attention layer we will follow \cite{levine2020limits} and set the inputs embedding matrix and the first layer self-attention key and query weights to:
	\begin{align*}		
		\left(M_{\textrm{V}}\right)_{\alpha,i} & =\begin{cases}
			1 & (\alpha - 1) \bmod d_a = d_a - 1 \\
			1 & 1 < i \le r - H \wedge \phi(\alpha) = i - 1 \\
			0 & \textrm{Otherwise}
		\end{cases}\\
		W_{i,j}^{K,1,h}&=W_{i,j}^{Q,1,h}=1_{i=1 \wedge j=d_{a}}
	\end{align*}		
	Clearly, the rank of $M_{\textrm{V}}\in\R^{d_x\times V}$ is (at most) $r$, since it has less than $r$ non zero rows, and thus satisfy the assumption regarding the embedding rank in theorem~\ifdefined\NOBODY 2\else\ref{theorem:embedding_rank_bottleneck_lower_bound}\fi ~of the main text.
	This assignment implements summation of the inputs embedding in the first self-attention layer since:
	\begin{align}
		\y^{(1,i)}(\hat{\mathbf{w}}^{(d_1)},\ldots,\hat{\mathbf{w}}^{(d_N)})_\alpha &= \sum_{j=1}^N \sum_{h=1}^H \left\langle W^{Q,1,h} M_{\textrm{V}}\hat{\mathbf{w}}^{(d_i)}, W^{K,1,h} M_{\textrm{V}}\hat{\mathbf{w}}^{(d_j)} \right\rangle W^{O,1,h} W^{V,1,h} M_{\textrm{V}}\hat{\mathbf{w}}^{(d_j)} \\
		&\overset{1}{=} \sum_{j=1}^N \sum_{h=1}^H  \overbrace{\left(M_{\textrm{V}}\right)_{d_a,d_i}}^{=1} \cdot \overbrace{\left(M_{\textrm{V}}\right)_{d_a,d_j}}^{=1} W^{O,1,h} W^{V,1,h} M_{\textrm{V}}\hat{\mathbf{w}}^{(d_j)} \\
		&\overset{2}{=} \left(\sum_{h=1}^H W^{O,1,h} W^{V,1,h} \right) \left(\sum_{j=1}^{N}M_{\textrm{V}}\hat{\mathbf{w}}^{(d_j)}\right)		
	\end{align}
	where $(1)$ is because $W^{Q,1,h} = W^{K,1,h}$ are matrices that are zero everywhere except for entry $(1,d_a)$, and $(2)$ because of linearity.	
	Therefore, for any $j_1,j_2\in\left[\multiset{\nicefrac{\left(r-1-H\right)}{2}}{3^{L-2}}\right]$ the output of the first self-attention layer on the sequence defined by 
	$\pi_P\left(j_1\right),\pi_Q\left(j_2\right)$ is:
	\begin{align}
		\label{equation:summation_assigmnet_first_layer_output}
		\y^{(1,j)}=\left(\sum_{h=1}^H W^{O,1,h} W^{V,1,h} \right)\underbrace{\sum_{t=1}^{\nicefrac{N}{2}}\left(M_{\textrm{V}}\hat{\mathbf{w}}^{\left({\pi_P\left(j_1\right)_t}\right)}+M_{\textrm{V}}\hat{\mathbf{w}}^{\left(\pi_Q\left(j_2\right)_t\right)}\right)}_{\eqqcolon\uu}
	\end{align}
	where $\left(p_t\right)_{t=1}^{\nicefrac{N}{2}}\in P$,$\left(q_t\right)_{t=1}^{\nicefrac{N}{2}}\in Q$ is some ordering of $P$ and $Q$.
	
	Denote by $E\coloneqq\max_{j,\alpha}\left(A_{j,\alpha}\right)$ the maximum entry of $A$ and let $N\ge E\cdot\left(r-1-H\right)$. Conceptually the mappings $\pi_P,\pi_Q$ will divide $P,Q$ into $\nicefrac{\left(r-1-H\right)}{2}$ length $E$ non-overlapping segments, where the $\alpha$`th segment will repeat $\hat{\mathbf{w}}^{\left(\alpha+1\right)}\quad A_{j,\alpha}$ times and fill the rest with the "zero" template vector $\hat{\mathbf{w}}^{\left(1\right)}$. Thus after the first self-attention layer summation we will get the relevant $A$'s rows.
	
	Formally, we define the mappings $\pi_P,\pi_Q$ as:
	\begin{align}
		&\forall 	j\in\left[\multiset{\nicefrac{\left(r-1-H\right)}{2}}{3^{L-2}}\right],t\in\left[\nicefrac{N}{2}\right]\quad\left(\pi_P\left(j\right)\right)_{p_t}=\begin{cases}
			\left\lfloor\nicefrac{t}{E}\right\rfloor + 2 & (t - 1) \bmod E < A_{j,\left\lfloor\nicefrac{t}{E}\right\rfloor+1} \\
			1 & \textrm{Otherwise}
		\end{cases}\\
		&\forall 	j\in\left[\multiset{\nicefrac{\left(r-1-H\right)}{2}}{3^{L-2}}\right],t\in\left[\nicefrac{N}{2}\right]\quad\left(\pi_Q\left(j\right)\right)_{q_t}=\begin{cases}
			\left\lfloor\nicefrac{t}{E}\right\rfloor + 2 + \nicefrac{\left(r-1-H\right)}{2} & (t - 1) \bmod E < A_{j,\left\lfloor\nicefrac{t}{E}\right\rfloor+1} \\
			1 & \textrm{Otherwise}
		\end{cases}
	\end{align}
	
	Finally, substituting $\pi_P,\pi_Q$ and $M_{\textrm{V}}$ in eq~\ref{equation:summation_assigmnet_first_layer_output} give the desired $\uu$:
	\begin{align*}
		\forall \alpha\in\left[d_x\right]\quad\uu_\alpha=\begin{cases}
			A_{j_1,\phi(\alpha)} & (\alpha-1)\bmod d_{a}<\frac{d_{a}-1}{2}\wedge\phi(\alpha)\le \nicefrac{\left(r-1-H\right)}{2}\\
			A_{j_2,\phi\left(\alpha-\frac{d_{a}-1}{2}\right)} & \frac{d_{a}-1}{2}\leq(\alpha-1)\bmod d_{a}<d_{a}-1\wedge\phi(\alpha-\frac{d_{a}-1}{2})\le \nicefrac{\left(r-1-H\right)}{2}\\
			N & (\alpha-1)\bmod d_{a}=d_{a}-1\\
			0 & \text{Otherwise}
		\end{cases}
	\end{align*}
	
\end{proof}

\subsection{Technical lemmas}

The lemma below prove the existence of matrix $A\in \N^{\multiset{d}{\lambda} \times d}$ with constant $l^2$ row norms, such that the operation of taking the rank $d$ matrix $AA^\top$ to the Hadamard power of $\lambda$ would result in a fully ranked matrix. Together with corollary~\ref{corollary:sufficient_assigmnet_layer_1}, this lemma is used in lemma~\ref{lemma:vocab_large_N_assigmnet} to prove theorem~\ifdefined\NOBODY 2\else\ref{theorem:embedding_rank_bottleneck_lower_bound}\fi ~of the main text for the vocabulary based embedding when assuming large $N$. Note that above lemma is an extension of a direct corollary of the proof in \citet{levine2020limits} regarding composition of the self-attention separation rank. 
\begin{lemma}\label{lemma:hadamard_power_rank}
	For any $d,\lambda\in \N$~there exist $A\in \N^{\multiset{d}{\lambda} \times d}$ with constant $l^2$ rows norm $c\in \N$ such that:
	\begin{equation}
		\rank{\left(AA^\top\right)^{\odot\lambda}}=\multiset{d}{\lambda}
	\end{equation}
\end{lemma}
\begin{proof}
	We will use the fact that
	\begin{equation}\label{eq:schmatic}
		\left(AA^\top\right)^{\odot\lambda}=\sum_{k=1}^{\multiset{d}{\lambda}} \aaa^{(k)} \otimes \bb^{(k)}
	\end{equation}
	is of full rank for $\{\aaa^{(k)}\}_{k=1}^{\multiset{d}{\lambda}}$ and $\{\bb^{(k)}\}_{k=1}^{\multiset{d}{\lambda}}$ which are two	sets of linearly independent vectors.
	
	For $\alpha,\beta\in[\multiset{d}{\lambda}]$, observing an entry of  $\left(AA^\top\right)^{\odot\lambda}$:
	\begin{align}
		\ensuremath{\left(\left(AA^{\top}\right)^{\odot\lambda}\right)_{\alpha\beta}=}&\left(AA^{\top}\right)_{\alpha\beta}^{\lambda}=\left(\sum_{r=1}^{d}v_{r}^{\left(\alpha\right)}v_{r}^{\left(\beta\right)}\right)^{\lambda}=\\\label{eq:multinomial}\sum_{k_{1}+\cdots+k_{d}=\lambda}\left(\begin{matrix}\lambda\\
			k_{1},\ldots,k_{d}
		\end{matrix}\right)&\left[\prod_{r=1}^{d}\left(v_{r}^{\left(\alpha\right)}\right)^{k_{r}}\right]\left[\prod_{r=1}^{d}\left(v_{r}^{\left(\beta\right)}\right)^{k_{r}}\right]
	\end{align}	
	where the first equality follows from the definition of the Hadamard power, in the section we denoted $v_{r}^{\left(\alpha\right)},v_{r}^{\left(\beta\right)}$ as the $r$th entries in rows $\alpha$ and $\beta$ of $A$, and in the second line we expanded the power with the multinomial identity.
	
	Identifying the form of eq.~\eqref{eq:multinomial} with the schematic form of eq.~\eqref{eq:schmatic}, it remains to find a specific matrix $A \in \N^{\multiset{d}{\lambda} \times d}$ with constant $l^2$ row norms $c\in \N$ for which the size $\multiset{d}{\lambda}$ set $\left\{\aaa^{(k_{1},\ldots,k_{d})}\right\}_{k_{1}+\cdots+k_{d}=\lambda}$ is linearly independent, where $a_\alpha^{(k_{1},\ldots,k_{d})}= \prod_{r=1}^{d}\left(v_{r}^{\left(\alpha\right)}\right)^{k_{r}}$.
	
	\citet{levine2020limits} proved there exists such  $B \in \mathbf{\R}^{\multiset{d}{\lambda} \times d}$ with $\forall \alpha,\beta\, \left[B\right]_{\alpha,\beta}>0$. 
	Therefore, it is enough to prove that we can approximate $B$ with non-negative rational\footnote{Given such non-negative rational matrix with normalized rows, we can multiply it by the common denominator
		and get the required $A\in \N^{\multiset{d}{\lambda} \times d}$ with constant $l^2$ rows norm.} matrix with normalized rows, while keeping the set $\left\{\aaa^{(k_{1},\ldots,k_{d})}\right\}_{k_{1}+\cdots+k_{d}=\lambda}$ linearly independent. 
	
	To prove this we will arrange the set as the columns of the matrix $C$, than $\left\{\aaa^{(k_{1},\ldots,k_{d})}\right\}_{k_{1}+\cdots+k_{d}=\lambda}$ is linearly independent if and only if $C$'s determinant is not zero. Now, $C$'s determinant is polynomial in $B$ entries and therefore from continuity arguments non-zero at neighborhood of $B$. Finally, this neighborhood contains row normalized rational matrix, since the unit sphere has a dense set of points with rational coordinates~\cite{schmutz2008rational}.
	
\end{proof}

\section{Experimental details}\label{section:experimental_details}
We conducted the network training described in section~\ifdefined\NOBODY 5\else\ref{sec:experiments}\fi ~of the main text with AdamW optimizer with $\beta_1=0.9,\beta_2=0.999$ and weight decay of $0.01$ for $1M$ steps and a batch size of $512$ sequences of $128$ tokens. All experiments used a learning rate schedule with a $12000$ step linear warm-up into $1.6\cdot10^{-3}$ followed by a cosine decay to zero and dropout rate of $0.1$. 
In order to increase width without changing other architectural parameters, for all the experiments except of section~\ifdefined\NOBODY 5.3\else\ref{sec:experiments:t5_bottleneck}\fi ~of the main text we kept the number of heads per layer constant at $2$ (experimental evidence indicates that many heads per layer are not crucial~\citep{michel2019sixteen,kaplan2020scaling}, as does \cite{levine2020limits} theoretical analysis which shows that the number of heads per layer affects the separation rank logarithmically).

To verify that our training recipe works, and the model differences are mainly due to expressiveness rather than optimization issues, we constructed a held out test-set of $100$ documents from OpenWebText and compare ourselves to GPT-2~\cite{radford2019language} published models. Table~\ref{table:owt_perplexity}~shows our models are on par with the GPT-2 models. Note that our models use shorter context of $128$ and that GPT-2 might have trained on our test set, which might explain it's superior perplexity. Nevertheless, this results shows that our training recipe is competitive, and that the model comparisons in the paper are indeed meaningful.
\begin{table}[ht]
	\ifdefined\SQUEEZE \vspace{-4mm} \fi
	\vskip 0.15in
	\begin{center}
		\begin{small}
			\begin{sc}
				\begin{tabular}{lcc}
					\toprule
					Model Size & GPT-2 Perplexity& Our Perplexity \\
					\midrule
					$117M$    & $21.11$ & $22.78$ \\
					$345M$    & $16.03$ & - \\
					$378M$    & -  & $17.95$ \\
					\bottomrule
				\end{tabular}
			\end{sc}
		\end{small}
	\end{center}
	\caption{OpenWebText test set perplexity where total number of tokens (98538 tokens) is according to gpt-2 standard vocabulary, and evaluation done with stride $1$ \ie for each token predication, the model use full context of the previous $N$ tokens.}	
	\label{table:owt_perplexity}
	\vskip -0.1in
	\ifdefined\SQUEEZE \vspace{-4mm} \fi
\end{table}

\subsection{Rank bottleneck degrades performance}
We conducted the network training described in subsection~\ifdefined\NOBODY 5.1\else\ref{sec:experiments:absoulte_degradation_exp}\fi ~of the main text width depth $L=12$ models.
The baselines widths are: $576,592,640,668,670,672,674,676,678,680,688$.
For the low-rank models we factorize the tokens and positional embedding into two matrices of dimensions $d_x \times r$, $r \times V$, as described in the main text.
The width in all of this model was set to $680$ and the $r$'s were: $16, 32, 64, 96, 128, 256, 344,400,456,512,600,680,880,1080$.
Table ~\ref{table:rank_bottleneck_std} show the estimated the standard deviation of the test loss of this experiment by repeating the training $5$ times.
For $r\in\left\{16,32,64\right\}$ we also trained variant with full-rank positional embedding and achieve losses that are within 2-std of the factorized positional embedding ones.
\begin{table}[ht]
	\caption{The standard deviation of the test loss for several experiments of subsection~5.1 in the main text, when repeating the training and evaluation experiment $5$ times per point.}
	\label{table:rank_bottleneck_std}
	\vskip 0.15in
	\begin{center}
		\begin{small}
			\begin{sc}
				\begin{tabular}{lcccr}
					\toprule
					$d_x$ & $r$ & std \\
					\midrule
					$680$    & $680$ & $8\cdot 10^{-4}$ \\
					$576$ & $576$ & $1.5\cdot 10^{-3}$\\
					$680$    & $128$  & $2.1\cdot 10^{-3}$ \\
					\bottomrule
				\end{tabular}
			\end{sc}
		\end{small}
	\end{center}
	\vskip -0.1in
\end{table}
\subsection{Vocabulary affects the depth-to-width interplay}
We tokenized the training and test corpus the with the GPT-2 \cite{radford2019language} vocabulary, and additional $3$ BPE vocabularies of sizes $V = 257,500,2000$  trained on our training corpus using to huggingface tokenizers\footnote{https://huggingface.co/docs/tokenizers/python/latest/} library.

We conducted the network training described in section~\ifdefined\NOBODY 5.2\else\ref{sec:experiments:implications_to_depth_width_interplay}\fi ~of the main text width depth $L\in\left\{24,48\right\}$ models with width detailed in table ~\ref{table:vocabs_depth_24_48_details}. 

\begin{table}[ht]
	\caption{The widths $d_x$ of the different trained networks.}
	\label{table:vocabs_depth_24_48_details}
	\vskip 0.15in
	\begin{center}
		\begin{small}
			\begin{sc}
				\begin{tabular}{lcccr}
					\toprule
					$V$ &$L$ & $r$ & widths   \\
					\midrule
					257 & 24    &  & $144, 160, 168, 184, 192, 200, 224, 248, 264, 280, 336, 408, 480$ \\
					257 & 48 & & $104, 112, 120, 128, 136, 142, 144, 160, 176, 184, 200, 240, 288, 336$ \\
					500 & 24 & &  $280, 336, 360, 384, 408, 424, 440, 480, 504, 528, 544, 576, 600$\\
					500 & 48 & & $200, 240, 272, 288, 296, 312, 336, 352, 376, 384, 408, 424$ \\
					2000 & 24 & $500$  & $280, 336, 408, 448, 480, 504, 528, 544, 576, 600, 628$ \\
					2000 & 48 & $500$  & $200, 240, 288, 320, 336, 352, 376, 384, 408, 424, 440$ \\
					2000  & 24  &   & $200, 224, 248, 280, 336, 408, 480, 544, 704, 744, 792, 848, 1064$ \\
					2000  & 48  &   & $144, 160, 176, 200, 240, 288, 336, 384, 496, 528, 560, 600, 752$ \\
					50257  & 24  &   & $408, 480, 544, 592, 656, 704, 744, 792, 848, 1064$ \\
					50257  & 48  &   & $336, 384, 432, 480, 512, 544, 576, 616, 768$ \\
					\bottomrule
				\end{tabular}
			\end{sc}
		\end{small}
	\end{center}
	\vskip -0.1in
\end{table}

\begin{table}[ht]
	\caption{The standard deviation of the test loss for several experiments of subsection~5.2 in the main text, when repeating the training and evaluation experiment $5$ times per point.}
	\label{table:vocabs_depth_24_vs_48_std}
	\vskip 0.15in
	\begin{center}
		\begin{small}
			\begin{sc}
				\begin{tabular}{lcccr}
					\toprule
					$V$ & $L$ & $d_x$ & $r$ & std \\
					\midrule
					$257$    & $24$ & $264$ & & $5.1\cdot 10^{-4}$ \\
					$257$    & $48$ & $184$ & & $6.8\cdot 10^{-4}$\\
					$500$    & $24$ & $504$ & & $1.7\cdot 10^{-3}$ \\
					$2000$   & $24$ & $528$ & $500$ & $1.5\cdot 10^{-3}$ \\
					$50257$   & $48$ & $480$ &  & $3.1\cdot 10^{-3}$ \\
					\bottomrule
				\end{tabular}
			\end{sc}
		\end{small}
	\end{center}
	\vskip -0.1in
\end{table}

Beyond the experiments described in subsection~\ifdefined\NOBODY 5.2\else\ref{sec:experiments:implications_to_depth_width_interplay}\fi ~of the main text, we conduct additional experiment to verify that when the vocabulary size exceed the network width and
does not constitutes a bottleneck, it has negligible effect on the “depth-efficiency” point. 

Figure~\ref{fig:depth_24_vs_48_vocabs_2000_50k} shows that when repeating subsection~\ifdefined\NOBODY 5.2\else\ref{sec:experiments:implications_to_depth_width_interplay}\fi ~experiment with GPT-2 \cite{radford2019language} vocabulary that is $\sim25$ times larger, the “depth-efficiency” point are very similar to the $V=2000$ case. This is directly in line
with the vocabulary bottleneck prediction since the largest network width in this experiment is $1064$, and clearly the $V=2000$ vocabulary does not constitutes a width bottleneck.
\begin{figure}[t]
	\vskip 0.2in
	\begin{center}
		\centerline{\includegraphics[width=1\columnwidth]{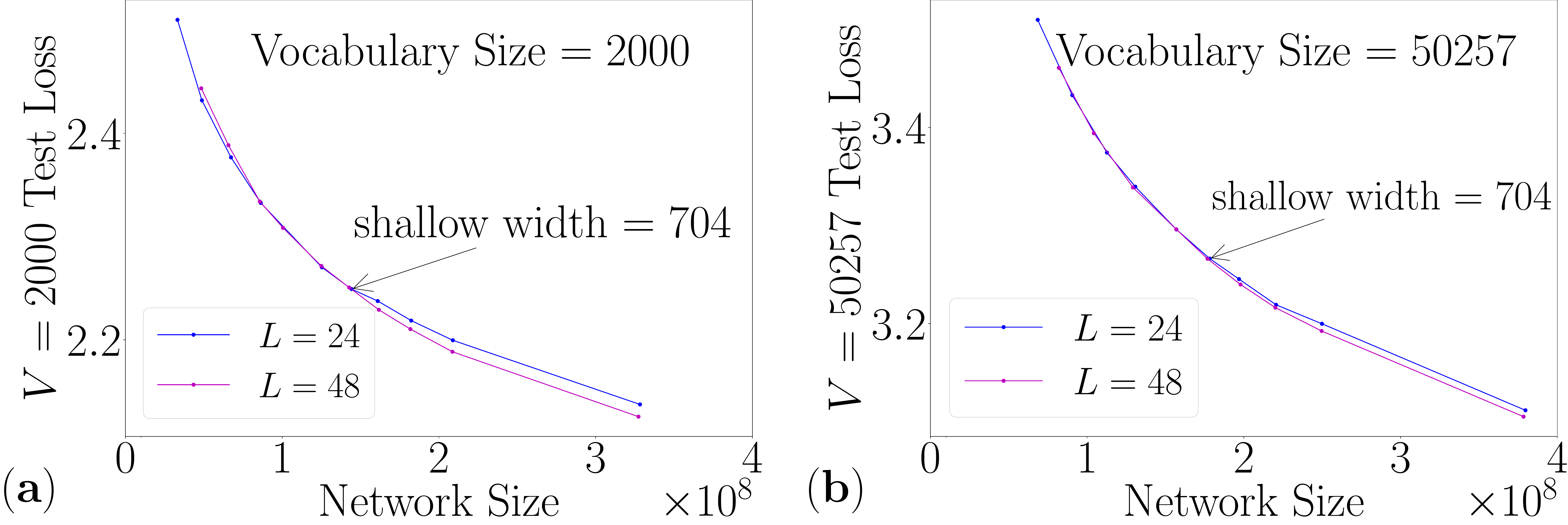}}
		\caption{Unlike figure 1 of the main text, when the vocabulary size exceed the network width and does not constitutes a bottleneck, the vocabulary size has negligible effect on the ``depth-efficiency" point . Note that the ``depth-efficiency" point of $V=50257$ occur at larger network size since when $V$ grows the embedding matrix size become non negligible and correspond to $20\%$ of the network parameters. Nevertheless, the shallow network width at the ``depth-efficiency" point is very similar to the $V=2000$ case.}
		\label{fig:depth_24_vs_48_vocabs_2000_50k}
	\end{center}
	\vskip -0.2in
\end{figure}
\subsection{Width bottlenecks the attention dimension}
Unlike the rest of the experiments in this paper, the experiments described in subsection~\ifdefined\NOBODY 5.3\else\ref{sec:experiments:t5_bottleneck}\fi ~of the main text done with larger amount of heads, mostly $H=12$ ~to avoid low-rank Key, Query, Value and
Output matrices when increasing the $\nicefrac{H\cdot d_a}{d_x}$ ratio.	Since the and the optimal depth per network might vary when changing  $\nicefrac{H\cdot d_a}{d_x}$ ratio we choose for each ratio the best depth in $\left\{12,18,24\right\}$. Table~\ref{table:tf_main_text_all_archs} give the exact details of the networks that appear in figure~\ifdefined\NOBODY 4\else\ref{fig:t5_expansion}\fi ~of the main text. Figure~\ref{fig:t5_depth_effect_is_minor} shows that the performance difference between values of the bottleneck ratio is	larger than the variation between different depths per bottleneck ratio.

\begin{figure}[hh]
	\vskip 0.2in
	\begin{center}
		\centerline{\includegraphics[width=0.5\columnwidth]{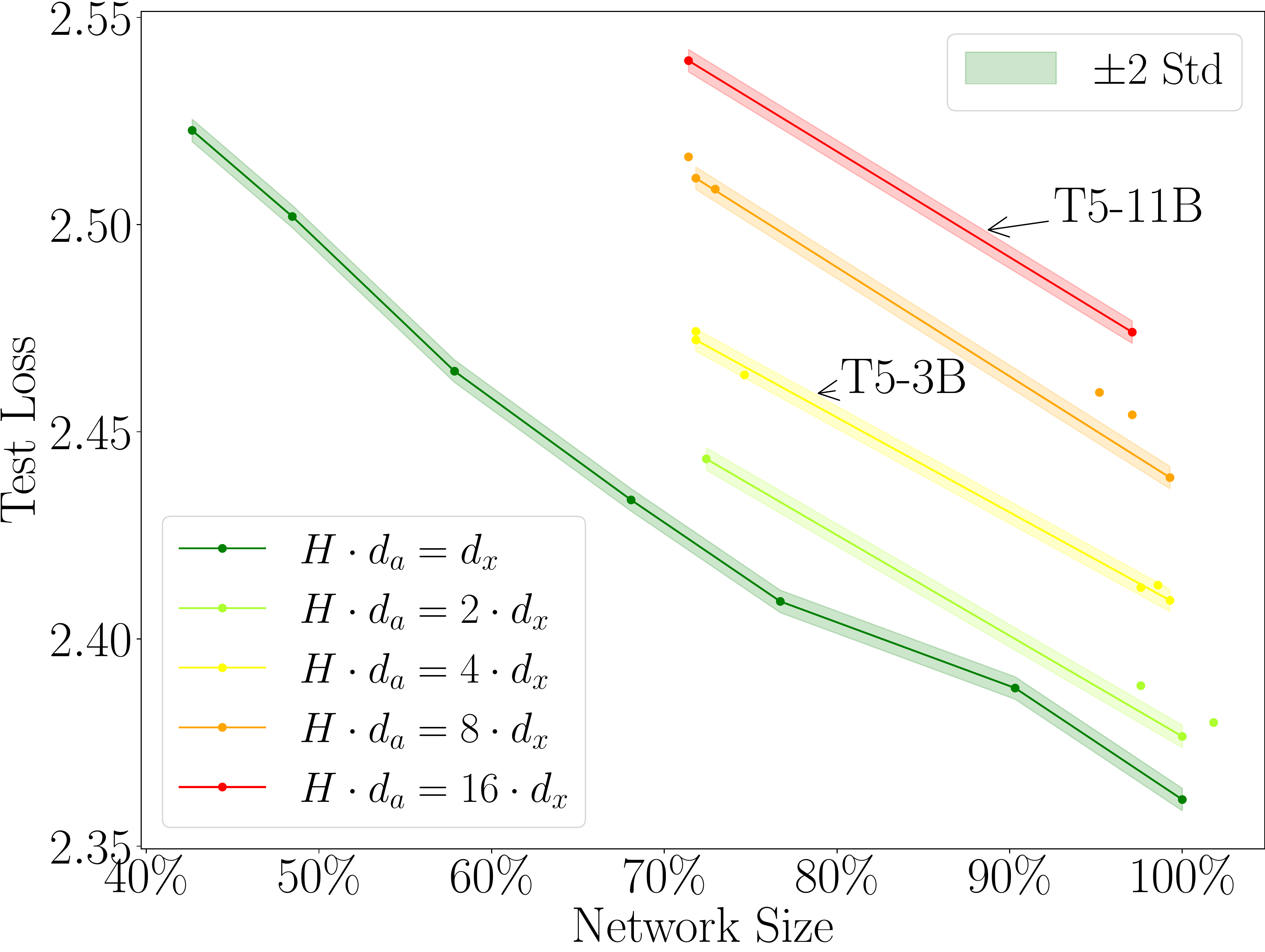}}
		\caption{The performance difference between values of the bottleneck ratio is larger than the variation between different depths per bottleneck ratio. The lines are the points of figure~\ifdefined\NOBODY 4\else\ref{fig:t5_expansion}\fi ~in the main text, and the circles are another depths detailed in table~\ref{table:tf_main_text_all_archs}.}
		\label{fig:t5_depth_effect_is_minor}
	\end{center}
	\vskip -0.2in
\end{figure}

\begin{table}[ht]
	\caption{Details of all architecture for each $\nicefrac{H\cdot d_a}{d_x}$ ratio in figure~\ref{fig:t5_depth_effect_is_minor}, the bold depth are the ones showed in figure~\ifdefined\NOBODY 4\else\ref{fig:t5_expansion}\fi ~of the main text.}
	\label{table:tf_main_text_all_archs}
	\vskip 0.15in
	\begin{center}
		\begin{small}
			\begin{sc}
				\begin{tabular}{lcccr}
					\toprule
					$\nicefrac{H\cdot d_a}{d_x}$ & $L$ & $H$ & widths   \\
					\midrule
					$1$ & $\mathbf{18}$ & $12$ & $360,384,420,456$ \\
					$1$ & $\mathbf{24}$ & $12$ & $420,456,480$ \\
					$2$ & $\mathbf{12}$ & $12$ & $408,480$ \\
					$2$ & $18$ & $18$ & $396$ \\
					$2$ & $24$ & $24$ & $336$ \\
					$4$ & $12$ & $12$ & $288,336$ \\
					$4$ & $18$ & $12$ & $240,276$ \\
					$4$ & $\mathbf{24}$ & $12$ & $204,240$ \\						
					$8$ & $\mathbf{12}$ & $12$ & $204,240$ \\
					$8$ & $18$ & $12$ & $168,192$ \\
					$8$ & $24$ & $12$ & $144,168$ \\
					$16$ & $\mathbf{12}$ & $24$ & $144,168$ \\
					\bottomrule
				\end{tabular}
			\end{sc}
		\end{small}
	\end{center}
	\vskip -0.1in
\end{table}

\subsection{Low-rank positional embedding}\label{sec:low_rank_positional_exps}
In this subsection, we show that a low rank positional embedding matrix has a negligible effect on the model loss, when compared with the effect of decreasing the vocabulary rank. 
This justifies both the practical use of rank-1 positional embedding matrices in leading models such as T5~\cite{raffel2019exploring}, and the assumption in theorems~\ref{theorem:vocab_uuper_bound} and~\ref{theorem:conv_uuper_bound}. 

We trained depth $L=12$ width $d_x=512$ networks with vocabulary size $V=2000$, sequences of $512$ tokens and positional embedding ranks of $\left\{1, 2, 4, 8, 16, 32, 64, 128, 256, 512\right\}$, In addition, we compared to baseline networks of vocabulary ranks\footnote{We did not compare to rank $1$ vocabulary network, since we observed dramatic performance degradation in this case that are probably due to unrelated bottleneck.} $\left\{2, 4, 8, 16, 32, 64, 128, 256, 512\right\}$ with full-rank positional embedding.

\begin{figure}[h]
	\vskip 0.2in
	\begin{center}
		\centerline{\includegraphics[width=0.5\columnwidth]{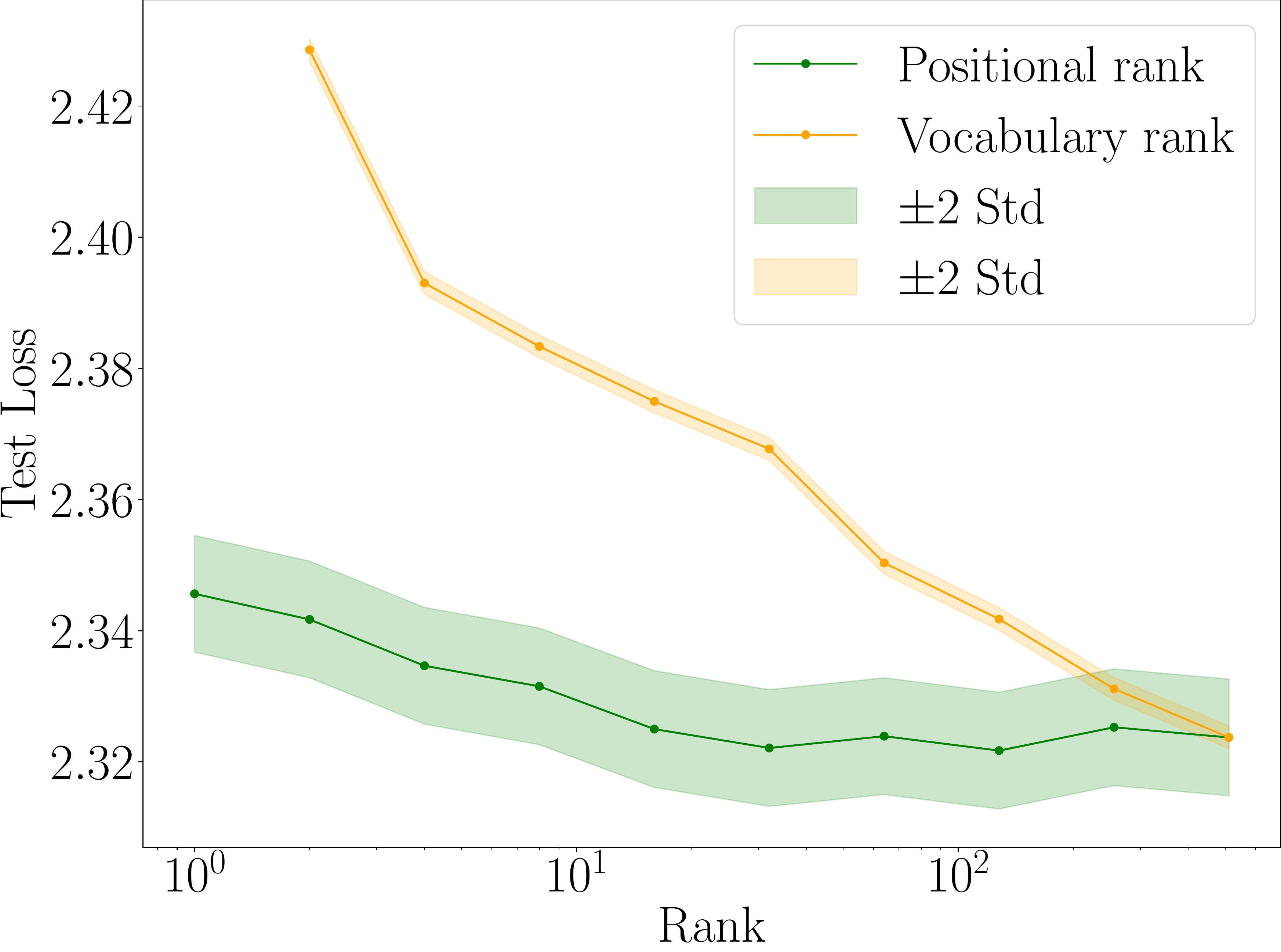}}
		\caption{Low-rank positional embedding as assumed in theorem~\ref{theorem:vocab_uuper_bound} when compared to low-rank due to vocabulary size has negligible effect on the model. For example even positional rank of $16$ perform on par with full positional embedding.}
		\label{fig:low_positional_rank}
		\ifdefined\SQUEEZE \vspace{-2mm} \fi
	\end{center}
	\vskip -0.2in
	\ifdefined\SQUEEZE \vspace{-3mm} \fi
\end{figure}

Figure~\ref{fig:low_positional_rank} shows that when decreasing the positional embedding rank down to $16$, performance is on par with full-rank positional embedding (within 2 std).
Moreover, even the extreme low-rank positional embedding of rank-$1$ reaches a loss of much higher vocabulary rank (between $64$ and $128$), thus justifying the assumption in theorems~\ref{theorem:vocab_uuper_bound} and~\ref{theorem:conv_uuper_bound}.  Note that we have shown in section~\ifdefined\NOBODY 5.1\else\ref{sec:experiments:absoulte_degradation_exp}\fi ~of the main text that the practical effect of  the predicted vocabulary bottlenecking  phenomenon is comparable to a substantial reduction in model size.